\documentclass[runningheads]{llncs}

\usepackage{eccv}

\usepackage{eccvabbrv}

\usepackage{graphicx}
\usepackage{booktabs}

\usepackage[accsupp]{axessibility}  %

\usepackage{color}
\usepackage{amssymb}
\usepackage{multirow, varwidth}
\usepackage{epsfig}
\usepackage{verbatim}
\makeatletter
\newif\if@restonecol
\makeatother

\usepackage{xr}
\usepackage{flushend}

\newcommand{\argmax}{\operatornamewithlimits{argmax}}
\newcommand{\argmin}{\operatornamewithlimits{argmin}}

\DeclareRobustCommand\onedot{\futurelet\@let@token\@onedot}
\def\onedot{.} %
\def\eg{\emph{e.g}\onedot, }

\def\etal{\emph{et al}\onedot}

\usepackage{bbm}
\usepackage{enumitem}
\usepackage{pifont}
\usepackage{mathtools}
\usepackage{thmtools}
\usepackage{wrapfig}
\usepackage{tabu}
\usepackage{multirow}
\usepackage{makecell}

\newcommand{\bftab}{\fontseries{b}\selectfont}

\usepackage{listings}
\lstnewenvironment{python}[1][]
{
    \lstset{
        language=Python,
        basicstyle=\small\ttfamily,
        commentstyle=\color{blue},
        #1
    }
}
{}

\newcommand{\tvec}[1]{\tilde{\mathbf{#1}}}
\newcommand{\vecx}{\mathbf{x}}
\newcommand{\tvecx}{\tvec{x}}

\usepackage{mathtools}
\newcommand{\defeq}{\vcentcolon=}
\newcommand{\mytilde}{\raise.17ex\hbox{$\scriptstyle\mathtt{\sim}$}}

\DeclareMathOperator*{\E}{\mathbb{E}}

\usepackage[linesnumbered,ruled,vlined]{algorithm2e}

\usepackage[]{todonotes} %

\usepackage{hyperref}

\usepackage{orcidlink}

\begin{document}

\title{Generalized Coverage for More Robust Low-Budget Active Learning} 

\author{Wonho Bae\inst{1} \and
Junhyug Noh\inst{2} \and
Danica J. Sutherland\inst{1,3}}

\authorrunning{W.~Bae et al.}
\institute{University of British Columbia, Vancouver \and 
Ewha Womans University \and
Alberta Machine Intelligence Institute (Amii) \\
\email{whbae@cs.ubc.ca},
\email{junhyug@ewha.ac.kr}, \email{dsuth@cs.ubc.ca}}

\maketitle

\begin{abstract}
    The ProbCover method of Yehuda et al.\ is a well-motivated algorithm for active learning in low-budget regimes, which attempts to ``cover'' the data distribution with balls of a given radius at selected data points. We demonstrate, however, that the performance of this algorithm is extremely sensitive to the choice of this radius hyper-parameter, and that tuning it is quite difficult, with the original heuristic frequently failing. We thus introduce (and theoretically motivate) a generalized notion of ``coverage,'' including ProbCover's objective as a special case, but also allowing smoother notions that are far more robust to hyper-parameter choice. We propose an efficient greedy method to optimize this coverage, generalizing ProbCover's algorithm; due to its close connection to kernel herding, we call it ``MaxHerding.'' The objective can also be optimized non-greedily through a variant of $k$-medoids, clarifying the relationship to other low-budget active learning methods. In comprehensive experiments, MaxHerding surpasses existing active learning methods across multiple low-budget image classification benchmarks, and does so with less computational cost than most competitive methods.
\end{abstract}

\section{Introduction}
\label{sec:intro}

In active learning,
rather than being handed a dataset with a random set of labeled points (as in passive learning),
a model strategically requests annotations for specific data points.
The aim is to obtain a better final model at a given label budget,
by requesting the ``most informative'' labels.
Some practical domains, such as medical imaging,
require reasonably-accurate models with a \emph{very} small set of labels~\cite{pixelpick2021Shin},
smaller than those used by most previous work in active learning (henceforth ``high-budget'' settings). 
This setting is more achievable than it used to be
thanks to good features from self-supervised learning techniques~\cite{simclr2020chen,scan2020van,dino2021caron};
with strong baseline features, actively trained models
can perform comparably to standard passive models trained on an order of magnitude more data~\cite{typiclust2022hacohen, probcover2022yehuda,low2022mahmood}.

Previous active learning techniques mostly select data points based on estimates of uncertainty, 
but as demonstrated by Hacohen~\etal\cite{typiclust2022hacohen},
in low-budget regimes, it can be far more effective to select representative data points than uncertain ones. 
To that end, Yehuda~\etal\cite{probcover2022yehuda} recently proposed ProbCover,
a greedy algorithm to maximize the \emph{coverage} of the training data.
Their notion of ``coverage,'' however, depends on a radius $\delta$;
we demonstrate that the performance of ProbCover is highly sensitive to the choice of $\delta$.

We thus introduce \textit{generalized coverage} (\cref{subsec:method:coverage}),
which encompasses ProbCover's objective as a special case,
while also allowing for smoother notions of coverage
which are far more robust to hyper-parameter selection.
We propose an algorithm named \textit{MaxHerding} (\cref{subsec:method:max_herding})
to greedily maximize the generalized coverage,
motivated by the data summarization algorithm called kernel herding~\cite{kherding2010chen}.
We also consider a non-greedy algorithm for the same objective (in \cref{subsec:method:non_greedy}),
a variant of $k$-medoids.
Doing so further clarifies the connection to other low-budget active learning approaches,
as depicted in \cref{fig:overview} and \cref{subsec:method:relationship}.

Our experiments (\cref{subsec:exp:sota}) show that MaxHerding outperforms existing active learning methods across various low-budget benchmark datasets, including ``imbalanced'' variants (\cref{subsec:exp:imbal}),
while using less computation cost than most competitive methods.
The non-greedy algorithm $k$-medoids marginally outperforms MaxHerding,
but does so at significantly increased computational cost (\cref{subsec:exp:ablation});
we thus recommend practitioners typically prefer MaxHerding for low-budget active learning.
To summarize our contributions:
\begin{itemize}
    \item We introduce a smoother and more robust notion of coverage, generalizing ProbCover's objective.
    \item We propose MaxHerding, a greedy algorithm to efficiently optimize generalized coverage, motivated by kernel herding.
    \item We illustrate that kernel $k$-medoids can also optimize this objective, helping clarify the relationship to other low-budget active learning methods.
    \item We demonstrate that MaxHerding outperforms existing active learning methods, with less computational cost than most competitive methods, while being more robust to the choice of budget size and features.
\end{itemize}

\section{Background}
\label{sec:background}

\subsection{Active Learning}
The aim of active learning is to reduce the need for manual annotation by strategically selecting data points for annotation. 
There are two major categories of active learning techniques:
uncertainty-based and representation-based methods.

The most common framework is pool-based active learning,
where we repeat the following process at each iteration $t \in [1,2, \dots, T]$:
\begin{enumerate}
    \item Train a pre-defined model on the labeled set $\mathcal{L}_t$ and evaluate it on a test set.
    \item Select a set of data points $\mathcal{S}_t = \{ \vecx_b \}_{b=1}^B$, where $B$ is a fixed budget, from an unlabeled set $\mathcal{U}_t = \{ \vecx_u \}_{u=1}^{\lvert \mathcal{U}_t \rvert}$.
    \item Annotate the selected set $\mathcal{S}_t$, merge it to a labeled set $\mathcal{L}_{t+1} = \mathcal{L}_{t} \cup \mathcal{S}_t$, and remove it from the unlabeled set $\mathcal{U}_{t+1} = \mathcal{U}_{t} \setminus \mathcal{S}_t $.
\end{enumerate}
For simplicity, we do not separately notate the labels $y$ for points in $\mathcal L$.

\subsubsection{Uncertainty-based}
Generally, uncertainty-based approaches leverage the model trained on $\mathcal L_t$ to find new points $\mathcal S_t$.
This approach includes simple yet effective algorithms like entropy~\cite{entropy2014wang}, margin sampling~\cite{margin2001scheffer}, and posterior probability~\cite{heterogeneous1994lewis,sequential1994lewis}.
These depend on the predictions from the current model, which does not take into account 
the potential impact of newly added data points.
\cite{al2009settles,badge2019ash} approximate the expected change in model parameters by measuring the change in loss gradients.
A related category of approaches
estimate expected reduction in error \cite{eer2001roy,eer_gf2003zhu,eer_mi2007guo} 
or expected change in model outputs \cite{emoc_reg2018kading,emoc2014frey,emoc2016kading,lookahead2022mohamadi}.

\subsubsection{Representation-based} The objective of this approach is to select examples from the unlabeled set that are considered the most ``typical'', expecting that performing effectively on these examples will generalize well to the entire unseen dataset.
Traditional methods are based on clustering algorithms such as $k$-means~\cite{kmeans2003xu} and variants such as $k$-means$^{++}$~\cite{badge2019ash}, $k$-medoids~\cite{kmedoids2016aghaee}, or $k$-medians~\cite{kmedian2012voevodski}.
Another option is determinantal point processes~\cite{dpp2012kulesza}.
There are also some hybrid methods that combine uncertainty and representation-based methods~\cite{hybrid_active_meta2021al,badge2019ash,alfa_mix2022parvaneh}.

In pool-based batch active learning, for typical values of the labeling budget $B T$,
uncertainty-based methods often outperform representation-based ones.

\subsection{Low-Budget Active Learning}
\label{subsec:bg:low_budget_al}

In many modern problems, however, settings where $B T$ is very small
(say, ten times the number of classes)
have become highly relevant.
In this low-budget regime,
popular uncertainty-based active learning methods such as entropy~\cite{entropy2014wang} and margin sampling~\cite{margin2001scheffer} are less effective even than random sampling~\cite{typiclust2022hacohen}:
the model simply is not good enough to ``know what it doesn't know.''
Thus, recent work on low-budget active learning
\cite{typiclust2022hacohen,probcover2022yehuda,low2022mahmood} instead selects data points which are {representative} of the unlabeled set
$\mathcal U$.

\subsubsection{Maximum coverage}

ProbCover~\cite{probcover2022yehuda} aims to select data points which maximally \emph{cover} the data distribution:
that is, it aims to have most of the data distribution lie within a distance at most $\delta$ from a labeled point.
This is motivated in part by their bound on the accuracy of a nearest neighbor (1-NN) classifier $\hat f$, compared to the true (deterministic) labeling function $f$:
\begin{equation}
\begin{split}
    \Pr_{\vecx}( f(\vecx) \ne \hat f(\vecx) )
    &\le
    \bigg( 1 - \Pr_{\vecx}\Big(\exists \vecx' \in \mathcal L \;\text{ s.t. }\; \lVert \vecx - \vecx' \rVert \le \delta \Big) \bigg)
    \\ 
    &+ \bigg( 1 - \Pr_{\vecx}\Big( \forall \vecx' \text{ s.t. } \lVert \vecx - \vecx' \rVert \le \delta, \;\; f(\vecx) = f(\vecx') \Big) \bigg)
.\end{split}
\label{eq:probcover_bound}
\end{equation}
The first probability on the right-hand side
is the \emph{coverage}, the probability that a randomly selected point is near a labeled point;
the second is the \emph{purity}, the probability that a point does not lie near a decision boundary.
Yehuda~\etal\ propose to choose a $\delta$%
\footnote{In low-budget settings, holding out a validation set for tuning hyperparameters such as $\delta$ may significantly reduce the training set size and severely degrade performance.}
such that the second term is acceptably small,
and then chooses $\mathcal L$ to minimize an empirical estimate of the first term.
Exact optimization is NP-hard,
and thus ProbCover greedily maximizes the estimated coverage~\cite{probcover2022yehuda}.

\subsubsection{Self-supervised features}
It is hard to learn good features based on a handful of labeled data points.
Previous work in low-budget active learning has thus used representations from self-supervised learning on the unlabeled data points, particularly with SimCLR~\cite{simclr2020chen}, which are then used in a 1-NN or linear classifier.
We follow previous works in focusing on SimCLR features, but also consider the more recent methods SCAN~\cite{scan2020van} and DINO~\cite{dino2021caron}.

\subsection{Kernel Herding}
\label{subsec:bg:kernel_herding}
Kernel herding~\cite{kherding2010chen} is a method to summarize an arbitrary probability distribution with a ``super-sample'':
a set of points such that empirical averages of $f$ on the super-sample
converge to the true expectation
faster than simple Monte Carlo averages over i.i.d.\ samples.
It does so by greedily aligning \emph{kernel mean embeddings} \cite{berlinet-thomas-agnan,Muandet_2017} computed in a reproducing kernel Hilbert space (RKHS) $\mathcal{H}$ with kernel $k$.
In our setting,
we aim to match the empirical distribution of the points $\mathcal U = \{ \vecx_n \}_{n=1}^{N}$.
After selecting $\{ \tilde\vecx_l \}_{l=1}^L$, the $(L+1)$-th chosen point is
\begin{align}
    \tilde\vecx_{L+1}
    &\in  \argmax_{\tilde{\vecx}} \underbrace{\frac{1}{N} \sum_{n=1}^N k(\vecx_n, \tilde{\vecx})}_\text{reward} - \underbrace{\frac{1}{L+1} \sum_{l=1}^L k(\tilde \vecx_l, \tilde{\vecx})}_\text{penalty} \label{eq:kherding_approx}
.\end{align}
The first ``reward'' term encourages $\tilde\vecx_{L+1}$ to be ``similar'' to the target points,
while the second ``penalty'' term discourages it from overlapping with already-selected data points.
Kernel herding can be seen \cite{herding_cond2012bach} as a type of Frank-Wolfe algorithm~\cite{frank1956algorithm},
also known as conditional gradient~\cite{cond_grad1966levitin};
it is also closely connected~\cite{wherding2012huszar}
to Bayesian quadrature~\cite{bq2003rasmussen}.
Stein Points~\cite{steinpoints2018chen} is a recent variant.

Herding converges much faster than the $\mathcal O(1 / \sqrt L)$ rate of i.i.d.\ sampling:
\begin{proposition}[Kernel herding convergence \cite{kherding2010chen}]
\label{thm:kernel-herding}
    Let $\vecx$ be distributed such that $k(\vecx, \vecx)$ is bounded for a positive definite kernel $k$,
    and assume a mild regularity condition which is typically satisfied as long as $\vecx$ is not concentrated at a single point.
    Then, for any function $g(\vecx)$ in the RKHS $\mathcal H$ with kernel $k$,
    \[ \left\lvert \frac1N \sum_{n=1}^N g(\vecx_n) - \frac1L \sum_{l=1}^L g(\tilde\vecx_l) \right\rvert = \mathcal O\left( \frac{\lVert g \rVert_{\mathcal H}}{L} \right) .\]
\end{proposition}

\begin{figure*}[t!]
    \centering
    \includegraphics[width=\linewidth]{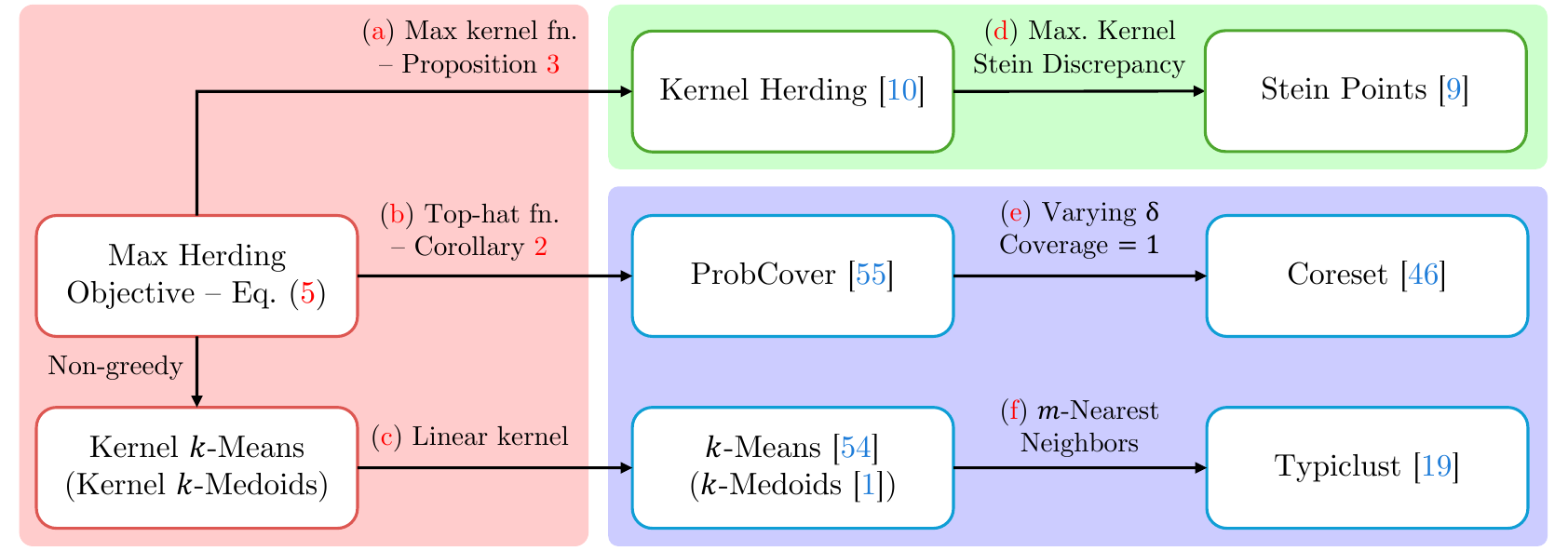}
    \caption{
    Illustration of the connection between our proposed methods (in red) and the existing active learning (in blue) as well as herding methods (in green).
    We describe (\textcolor{red}{a})--(\textcolor{red}{c}) in \cref{sec:method}. For (\textcolor{red}{d})--(\textcolor{red}{f}), please refer to \cite{steinpoints2018chen}, \cite{probcover2022yehuda}, and \cite{typiclust2022hacohen}, respectively.
    }
    \label{fig:overview}
\end{figure*}

Active learning was mentioned as an application by Chen \etal~\cite{kherding2010chen},
but this has not seen much uptake in practice.

\section{Our Method}
\label{sec:method}

\subsection{Generalized Coverage}
\label{subsec:method:coverage}
We begin with a more general version of the bound \eqref{eq:probcover_bound} of Yehuda  \etal~\cite{probcover2022yehuda}.

\begin{restatable}{theorem}{ourbound}
    Let $\hat f$ denote a 1-NN classifier trained on $\left\{ \big(\vecx_l, f(\vecx_l) \big) \right\}_{l=1}^{\lvert \mathcal L \rvert}$,
    for a deterministic labeling function $f$.
    Let $k$ be a nonnegative function
    such that for any $\vecx$,
    if $\lVert \vecx_1 - \vecx \rVert \le \lVert \vecx_2 - \vecx \rVert$,
    then $k(\vecx, \vecx_1) \ge k(\vecx, \vecx_2)$.
    Then
    \begin{equation}
        \Pr_{\vecx}\left( f(\vecx) \neq \hat{f}(\vecx) \right)
        \le
        \left( 1 - \E_{\vecx}\left[ \max_{\vecx' \in \mathcal{L}} k(\vecx, \vecx') \right] \right)
        + \E_{\vecx}\left[ \max_{\vecx' : f(\vecx) \ne f(\vecx')} k(\vecx, \vecx') \right]
        \label{eq:ourbound}
    .\end{equation}
    \label{thm:ourbound}
\end{restatable}
The proof is given in \cref{app:subsec:ourbound}.
This result is in fact a strict generalization of \eqref{eq:probcover_bound},
as shown in the following result (proved in \cref{app:subsec:bound}).
\begin{restatable}{corollary}{bound} \label{thm:gencov-tophat}
    Let $k$ be the top-hat function defined as follows for a given $\delta$:
    \[ k(\vecx, \vecx') \defeq \begin{cases}
        1 &\text{if } \lVert \vecx - \vecx' \rVert \le \delta \\
        0 &\text{otherwise }
        .\end{cases}
    \]
    Then the statement of \eqref{eq:ourbound} becomes identical to that of \eqref{eq:probcover_bound}.
    \label{thm:bound}
\end{restatable}

We define the \emph{generalized coverage} as the term in \eqref{eq:ourbound} corresponding to the coverage term of \eqref{eq:probcover_bound}.
Similarly to ProbCover, we propose to use a $k$ such that the second (impurity) term of \eqref{eq:ourbound} is small,
and then choose $\mathcal L$ to maximize the generalized coverage,
guaranteeing good downstream performance of $\hat f$.

\begin{definition}[Generalized Coverage]
For a labeled set $\mathcal{L}$ and a real-valued function $k$,
the \emph{generalized coverage} is defined as
    \begin{equation}
        \mathrm{C}_k(\mathcal{L}) \defeq \E_{\vecx} \left[ \max_{\vecx'\in\mathcal{L}} k(\vecx, \vecx') \right]
        \label{eq:gen-coverage}
    .\end{equation}
\end{definition}
We do not require $k$ to be positive definite (or integrate to one),
but typically use 
a Gaussian kernel $k(\vecx, \tvecx) = \exp\left( - \frac{1}{2\sigma^2} \lVert \vecx - \tvecx \rVert^2 \right)$.
Both the Gaussian kernel and the top-hat function satisfy the conditions of \cref{thm:ourbound}.

\subsubsection{Extension to a linear classifier}
The bounds so far have assumed a 1-NN classifier $\hat f$.
In practice, however, linear classifiers are more common.
In the low-budget regime where $n \ll d$,
gradient descent for logistic regression will find the max-margin separator~\cite{soudry2018implicit}.
Moreover, for high-dimensional normalized features,
it is plausible that a ``successful'' choice of $\mathcal L$ with one data point per class
will have nearly-orthogonal data points.
Lemma 1 of Bae~\etal\cite{active_meta2023bae} shows that in this case,
the max-margin linear separator in fact exactly agrees with the 1-NN classifier.
Thus, motivations based on 1-NN classifiers should give some intuition about the performance of linear classifiers as well.

\subsection{Greedy Maximization: MaxHerding}
\label{subsec:method:max_herding}

Our main algorithm is greedy maximization of $C_k(\mathcal L)$,
which we call \emph{MaxHerding}.

The generalized coverage is defined in terms of an expectation over $\vecx$,
but we have access only via samples.
Thus, we use the Monte Carlo approximation
\[
    C_k(\mathcal L)
    = \E_{\vecx}\left[ \max_{\vecx' \in \mathcal L} k(x, x') \right]
    \approx \frac1N \sum_{n=1}^N \left(\max_{\vecx' \in \mathcal L} k(\vecx_n, \vecx')\right)
    =: \hat C_k(\mathcal L)
.\]
If $a \le k(\vecx, \vecx') \le b$,
then
$\Pr\left( \lvert \hat C_k(\mathcal L) - C_k(\mathcal L) \rvert \le (b - a) \sqrt{\frac 2 N \log\frac2\delta} \right) \ge 1 - \delta$
by Hoeffding's inequality.
For the top-hat kernel of ProbCover, $b - a = 1$.
In active learning settings, $N$ (the size of the unlabeled pool) is typically very large, at least on the order of tens of thousands, so $\hat C_k$ will be very close to $C_k$.

The improvement in $\hat C_k(\mathcal L)$ from adding one more entry $\tvecx$ to $\mathcal L$ is
\begin{equation}
    \max_{\tilde\vecx} \left(
        \frac1N \sum_{n=1}^N \max_{\vecx' \in \mathcal L \cup \{\tilde\vecx\}} k(\vecx_n, \vecx')
        -
        \frac1N \sum_{n=1}^N \max_{\vecx' \in \mathcal L} k(\vecx_n, \vecx')
    \right)
    \label{eq:max_herding}
.\end{equation}
MaxHerding simply picks the corresponding maximizer $\tvecx$ from \eqref{eq:max_herding} at each step.

\begin{proposition} \label{thm:submodular-guarantee}
For $\mathcal L$ selected by MaxHerding with a nonnegative function $k$,
\begin{equation}
    \hat C_k(\mathcal L)
    \ge
        \left(1 - \frac1e\right)
        \max_{\mathcal S : \lvert \mathcal S \rvert \le \lvert \mathcal L \rvert} \hat C_k(\mathcal S)
\label{eq:submodular-guarantee}
.\end{equation}
\end{proposition}
\begin{proof}
The generalized coverage
can be written as $g(\mathcal{S}) \defeq \sum_{i=1}^N \max_{j \in \mathcal{S}} M_{i, j}$;
when all $M_{i,j} \ge 0$,
these \emph{facility location} problems are monotone submodular~\cite{plant_location1974frieze,facility_location1990mirchandani}.
The result follows from standard guarantees on the quality of greedy optimization of submodular functions~\cite{submodular_greedy1978nemhauser,submodular2014krause}.
\qed
\end{proof}

\subsubsection{Connection to ProbCover}
This result follows immediately from \cref{thm:gencov-tophat}.
\begin{restatable}{corollary}{probcover}
     MaxHerding with a top-hat function is identical to ProbCover.
    \label{thm:probcover}
\end{restatable}

\subsubsection{Connection to kernel herding}
The following result, proved in \cref{app:subsec:maxherding},
relates
MaxHerding
to kernel herding (thus explaining the name).
\begin{restatable}{proposition}{maxherding}
    Let the max kernel function for a labeled set $\mathcal L$
    be given by
    \[
        k(\vecx, \tvecx; \mathcal L)
        :=
        \max\left\{
            k(\vecx, \tvecx)
            - \max_{\vecx' \in \mathcal L} k(\vecx, \vecx')
        , 0 \right\}
    .\]
    Then we can rewrite \eqref{eq:max_herding}, the improvement in one step of MaxHerding, as
    \begin{equation} \label{eq:maxherding-reframed}
        \max_{\tvecx}
        \frac1N \sum_{n=1}^N k(\vecx_n, \tvecx; \mathcal L)
        - \frac{1}{\lvert \mathcal L \rvert + 1} \sum_{l=1}^{\lvert \mathcal L \rvert}
          k(\tvecx_l, \tvecx; \mathcal L)
    .\end{equation}
    \label{thm:maxherding}
\end{restatable}
Comparing to \eqref{eq:kherding_approx} shows that
MaxHerding
takes the same form as kernel herding
using a \emph{max kernel} which only counts the ``closest'' contribution,
rather than allowing multiple contributions from nearby points to ``add up.''

Kernel herding tries to match a target distribution,
and hence if many points are clustered in a small region,
it will select several points from a small area to match the probability density.
If we assume that the deterministic true labeling function $f$ is relatively smooth,
this effort is simply wasted in active learning:
we already know the local label from the first point in that area,
and our effort is probably better spent obtaining labels from elsewhere.

\begin{figure*}[t!]
    \centering
    \begin{subfigure}[b]{0.39\textwidth} %
        \includegraphics[width=\textwidth]{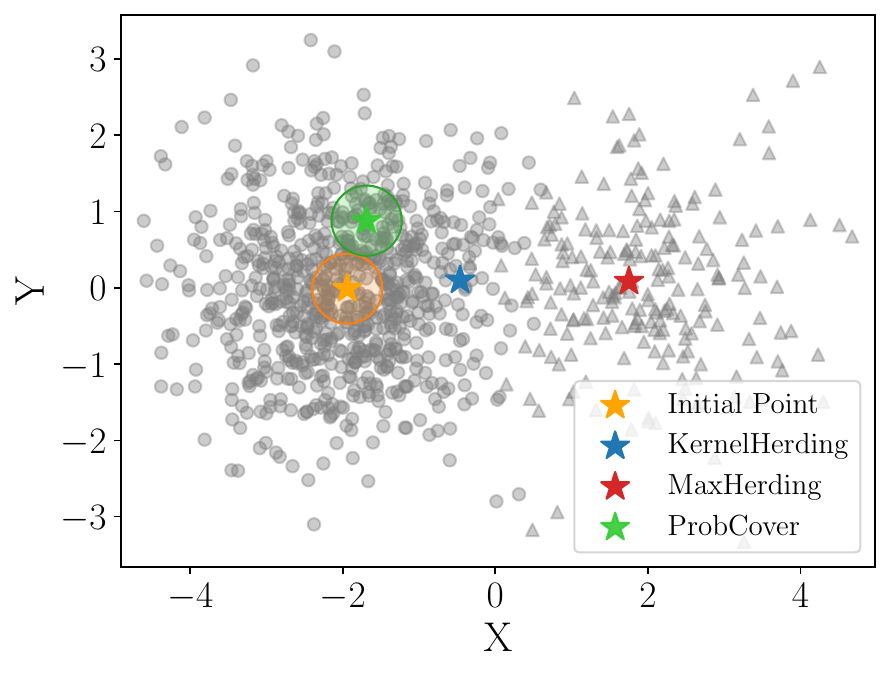}
        \caption{Comparison for the next choices}
        \label{fig:toy}
    \end{subfigure}
    \hfill
    \begin{subfigure}[b]{0.59\textwidth} %
        \includegraphics[width=\textwidth]{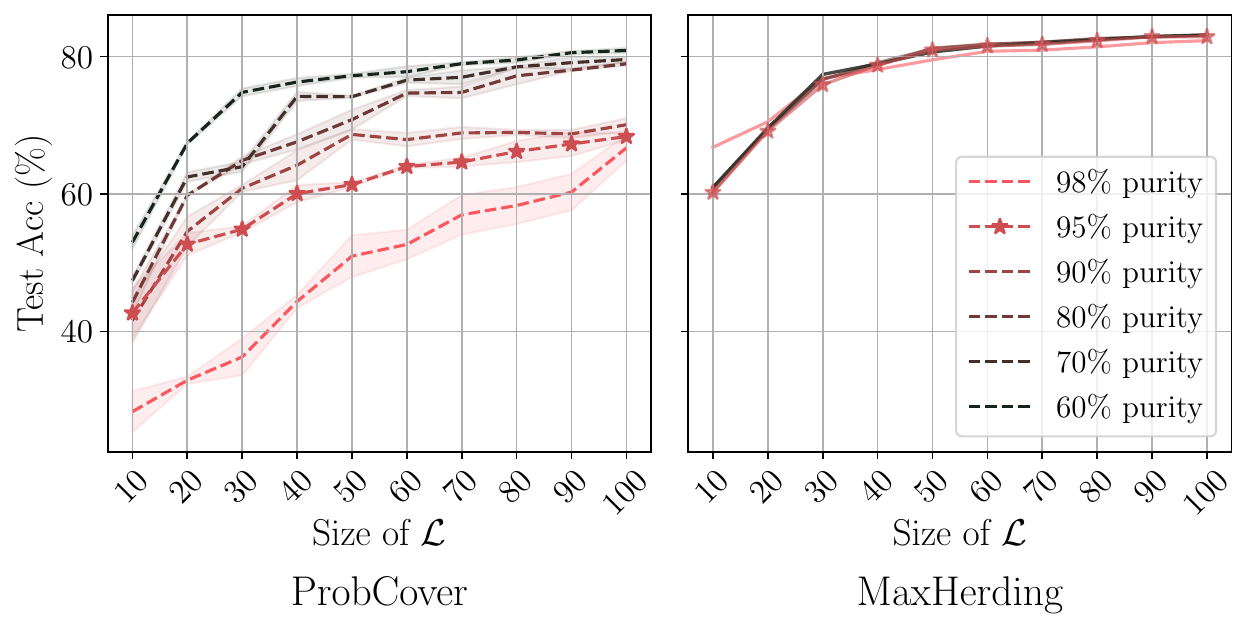}
        \caption{Sensitivity to $\delta$ (derived from purity)}
        \label{fig:sensitivity}
    \end{subfigure}
    \caption{(a) Next choices for different algorithms after selecting the initial orange point. (b) Varying the radius $\delta$ in ProbCover, and Gaussian lengthscale for MaxHerding, for CIFAR10 classification. See text for the definition of purity.}
    \label{fig:toy_and_sensitivity}
\end{figure*}

For an illustration, see \cref{fig:toy}, where 80\% of data points (marked as $\circ$) come from one Gaussian component,
and the remaining 20\% ($\vartriangle$) from another.
Kernel herding, MaxHerding, and ProbCover will each choose roughly the mean of the larger component as their first point (orange).
Freezing that selection, the figure shows the second point selected by each algorithm:
MaxHerding chooses roughly the mean of the second component,
and both ProbCover and kernel herding choose points within the first larger component.
ProbCover chooses a point just outside the radius of its initial selection;
kernel herding chooses a point far enough to have a small penalty,
but close enough to still exploit a large amount of reward from ``double-counting'' the larger cluster.
MaxHerding also shows improved behavior compared to these algorithms in our experiments.

The max kernel is not a positive definite kernel function (for one, it is not symmetric),
and so bounds for kernel herding such as \cref{thm:kernel-herding} do not automatically apply to MaxHerding.
It is worth noting, however, that \cref{thm:kernel-herding} is not actually particularly applicable to active learning settings either.

\begin{restatable}{brem}{kherding}
    Let $L_{\hat f}(\vecx) \defeq \E_y [\ell(\hat{f}(\vecx), y) \mid \vecx]$
    denote the expected loss on $\vecx$ for a predictor $\hat f$ at $\vecx$.
    (If $y$ is a deterministic function $f$ of $\vecx$ and $\ell$ is the $0$-$1$ loss corresponding to accuracy, $L_{\hat f}(\vecx)$ is $0$ if $\hat f(\vecx) = f(\vecx)$ and $1$ otherwise.)
    
    We wish to find the predictor $\hat f$ with minimal expected loss $L_{\hat f} = \E_{\vecx} L_{\hat f}(\vecx)$.
    \Cref{thm:kernel-herding} establishes that if $\mathcal L$ is chosen with kernel herding in a RKHS $\mathcal H$,
    our estimated loss $\hat L_{\hat f} = \frac1{\lvert \mathcal L \rvert} \sum_{l=1}^{\lvert \mathcal L \rvert} L_{\hat f}(\tvecx_l)$
    converges to the true value $L_{\hat f}$ at a rate of $\mathcal O\left( \lVert L_{\hat f} \rVert_{\mathcal H} / L \right)$,
    and thus we can expect based on uniform convergence
    that learning based on $\mathcal L$ will quickly identify the best predictor.

    For classification tasks with the $0$-$1$ loss and a continuous kernel, however,
    $\lVert L_{\hat f} \rVert_{\mathcal H} = \infty$
    \cite[Lemma 4.28]{steinwart-christmann}.
    Thus, the upper bound is infinite.
    \label{thm:kherding}
\end{restatable}

\subsubsection{Efficient implementation}
Using the expressions from \cref{thm:maxherding},
we can use dynamic programming to select each point in $\mathcal O(N)$ time,
rather than $\mathcal O( N \lvert \mathcal L \rvert )$ as for a direct implementation of \eqref{eq:max_herding}.
\Cref{alg:our_method} gives pseudocode.

\begin{algorithm}[t!]
\newlength{\commentWidth}
\setlength{\commentWidth}{7cm}
\newcommand{\atcp}[1]{\tcp*[r]{\makebox[\commentWidth]{#1\hfill}}}
\SetKwInput{KwInput}{Input}
\caption{Greedy and non-greedy methods for generalized coverage}\label{alg:our_method}
\KwInput{Labeled set $\mathcal{L}$, unlabeled set $\mathcal{U}$, budget $B$, number of iterations $T$ }
Compute $\mathbf{k} \in \mathbbm{R}^{\lvert \mathcal{U} \rvert}$ where $\mathbf{k}_i = \max_{\vecx' \in \mathcal{L}} k(\vecx_i, \vecx')$ \\
\For{$t \in [1, 2, \cdots, T]$}{
     \eIf(\tcp*[h]{\textcolor{blue}{Greedy method}}){ \textnormal{Use Greedy} }{
    \For{$b \in [1, 2, \cdots, B]$}{ 
        Select $\vecx_b^* = \argmax_{\tvecx \in \mathcal{U}} \frac{1}{N} \sum_{n=1}^N \max(k(\vecx_i, \tvecx) - \mathbf{k}_i, 0)$\\
        Update $\mathbf{k}_i \leftarrow \max (k(\vecx_i, \vecx_b^*), \mathbf{k}_i ) $
    }
    }(\tcp*[h]{\textcolor{blue}{Non-greedy method}}){
    Fit kernel $k$-medoids with $\lvert \mathcal{L} \rvert$ fixed and $B$ random initial centroids \\
    Select new $B$ centroids as $\{\vecx_1^*, \vecx_2^*, \cdots, \vecx_B^* \}$ \\
    }
    \tcp{\textcolor{blue}{For both greedy and non-greedy methods}}
    Update $\mathcal{L} \leftarrow \mathcal{L} \cup \{\vecx_b^*\}_{b=1}^B$ and $\mathcal{U} \leftarrow \mathcal{U} \,\backslash \, \{\vecx_b^*\}_{b=1}^B$ 
}
\end{algorithm}
\vspace{-3mm}

\subsubsection{Sensitivity to radius $\delta$} 
We argue that the top-hat function used by ProbCover is not a good choice for measuring similarity between data points.
Its notion of smoothness for data points is ``all or nothing,''
treating two points as exactly equivalent at distance $0.99\delta$ and utterly unconnected at $1.01 \delta$.
Thus, it is perhaps unsurprising that ProbCover is highly sensitive to the choice of $\delta$.

\Cref{fig:sensitivity}
shows downstream active learning performance on CIFAR10 for different values of $\delta$ (details in \cref{sbusec:exp:setup}).
Following Yehuda~\etal~\cite{probcover2022yehuda},
we select $\delta$ based on an estimate of its purity:
95\% purity means that the second term in the upper bound of
\eqref{eq:probcover_bound} is $0.05$,
as estimated on samples with $k$-means labels.
(The purity is a monotonically decreasing function of $\delta$.)

We can see that active learning performance for ProbCover is highly sensitive to the choice of $\delta$;
in this case, larger $\delta$ performs better, but this varies by dataset and features.
By contrast, varying the lengthscale parameter of a Gaussian kernel to set the purity term in \eqref{eq:ourbound} barely changes downstream performance here.
We thus stick to a default value of $1$, a sensible number for normalized features.
We also show that MaxHerding is robust to the choice of kernels in \cref{fig:kernels}.

\subsection{Non-Greedy Maximization: Kernel $k$-Medoids}
\label{subsec:method:non_greedy}

Given already-selected points $\mathcal L$,
the problem of choosing a new batch $\mathcal S$ to label
in order to maximize the generalized coverage $C_k$ of \eqref{eq:gen-coverage}
can be written as
\begin{align}
    \argmax_{\mathcal{S} \subset \mathcal{X}, \lvert \mathcal{S} \rvert = B }
    & \frac{1}{N} \sum_{n=1}^N \max_{x' \in \mathcal{L} \cup \mathcal{S}} k(\mathbf{x}_n, \mathbf{x}') \label{eq:non_greedy}
.\end{align}
\begin{proposition} \label{thm:k-medoids}
Let $k$ be a positive definite kernel
with feature map $\phi$, so that $k(\vecx, \vecx') = \langle \phi(\vecx), \phi(\vecx') \rangle_{\mathcal H}$
for some Hilbert space $\mathcal H$.
Suppose $k(\vecx, \vecx)$ is constant
(as holds, \eg for Gaussian kernels).
Then the maximizer of \eqref{eq:non_greedy} exactly agrees with the objective of $k$-means or $k$-medoids,
with cluster centers in $\mathcal L \cup \mathcal S$:
\begin{align}
    \eqref{eq:non_greedy}
    =
    \argmax_{\mathcal{S} \subset \mathcal{X}, \lvert \mathcal{S} \rvert = B } \frac{1}{N} \sum_{n=1}^N \min_{x' \in \mathcal{L} \cup \mathcal{S}} \lVert \phi(\vecx_n) - \phi(\vecx')\rVert_\mathcal{H}^2
    \label{eq:kmeans}
.\end{align}
\end{proposition}
\begin{proof}
    We have that 
    $\lVert \phi(\vecx_n) - \phi(\vecx')\rVert_\mathcal{H}^2
    = k(\vecx_n, \vecx_n) + k(\vecx', \vecx') - 2 k(\vecx_n, \vecx')$.
    \qed
\end{proof}
Using a linear kernel\footnote{If $\lVert \vecx \rVert = 1$, then the linear kernel satisfies our assumption of $k(\vecx, \vecx)$ being constant. $k$-means would not select points satisfying this assumption, but $k$-medoids would.} $k(\vecx, \vecx') = \vecx \cdot \vecx'$, we recover standard $k$-means.

In pool-based active learning, we cannot query the labels of arbitrary points;
we thus cannot simply query at the cluster centers obtained by $k$-means.
(Using kernel $k$-means and having $\vecx$ the output space of a featurizer network both also mean that arbitrary points in $\mathcal H$ may not correspond to \emph{any} possible input image.)
Typiclust~\cite{typiclust2022hacohen} thus runs $k$-means
(capping the number of clusters at $500$ to try to avoid empty clusters)
and chooses the ``densest'' point in each cluster to label.
Doing so, however, might not choose a ``central'' point in that cluster.
Given our generalized coverage framing,
we can instead directly enforce $\mathcal L \cup \mathcal S \subseteq \mathcal U$
and optimize the objective \eqref{eq:kmeans} with $k$-medoids,
as described in \cref{alg:our_method}.

The most popular approach for high-quality solutions to $k$-medoids problems is the
\textit{Partitioning Around Medoids} (PAM)~\cite{pam2009kaufman} algorithm.
Although we found implementations of ``plain'' PAM impractically slow, 
the FasterPAM algorithm of \cite{fastPAM2019schubert,fasterPAM2021schubert,fasterPAM2022schubert}
is more efficient.
We modified the fast Rust~\cite{rust2014matsakis} implementation
from Schubert and Lenssen~\cite{fasterPAM2022schubert}
to support ``freezing'' $\mathcal L$ while optimizing $\mathcal S$ in \eqref{eq:kmeans}.

In our experiments, kernel $k$-medoids achieves similar or slightly better downstream performance compared to MaxHerding, as shown in \cref{fig:greedy_vs_nongreedy}.
(We know from \cref{thm:submodular-guarantee} that the amount of possible improvement in $C_k$ is limited.)
Even this optimized implementation is much slower than other methods, however (see \cref{subsec:exp:ablation}).
Hence, we recommend practitioners default to MaxHerding over kernel $k$-medoids.

\subsection{Relationships with Existing Methods}
\label{subsec:method:relationship}

\Cref{fig:overview} describes connections between our approaches based on generalized coverage, $C_k(\mathcal L)$, to several existing active learning and herding algorithms.
(\textcolor{red}{a}, \cref{thm:maxherding}) MaxHerding is kernel herding \cite{kherding2010chen} with a ``max kernel function,''
also establishing (\textcolor{red}{d}) a connection to Stein Points~\cite{steinpoints2018chen}.
(\textcolor{red}{b}, \cref{thm:probcover}) ProbCover~\cite{probcover2022yehuda} is a special case of MaxHerding with a top-hat kernel function,
also establishing (\textcolor{red}{e}) a close connection to Coreset~\cite{coreset2017sener}.
(\textcolor{red}{c}, \cref{eq:kmeans}) Maximizing generalized coverage is equivalent to kernel $k$-medoids,
connecting (\textcolor{red}{f}) to Typiclust~\cite{typiclust2022hacohen} through $k$-means clustering.

\section{Experimental Evaluation}
\label{sec:experiments}
\label{sbusec:exp:setup}

\subsubsection{Active learning methods} We compare our proposed MaxHerding method to the following baselines.
Uncertainty, Entropy, and Margin are representative uncertainty-based methods.
Coreset, Typiclust, and ProbCover are the leading representation-based methods,
while KernelHerding was selected due to its connection to our algorithm.

\begin{description}[itemsep=2pt,parsep=2pt,leftmargin=!,labelindent=1em,labelwidth=5em,itemindent=4em,labelsep=1ex,]
    \item[Random] Uniformly select $B$ data points at random. 
    \item[Uncertainty] Select $\vecx^* = \argmin_{\tvecx \in \mathcal{U}} p_1(y|\tvecx)$, one point at a time, where $p_1$ denotes the highest predicted probability.
    \item[Entropy~\cite{entropy2014wang}]
    Select $\vecx^* = \argmax_{\tvecx \in \mathcal{U}} H(\hat y(\tvecx) \mid \tvecx)$, one point at a time, where $H(\cdot)$ is the Shannon entropy.
    \item[Margin~\cite{margin2001scheffer}] Select $x^* = \argmin_{\tvecx \in \mathcal{U}} p_1(y|\tvecx) - p_2(y|\tvecx)$, one point at a time, where $p_2$ denotes the second-highest predicted probability.
    
    \item[Coreset~\cite{coreset2017sener}] Select one point at a time using the $k$-Center-Greedy algorithm, defined as $\vecx^* = \argmax_{\tvecx\in\mathcal{U}} \min_{\vecx' \in \mathcal{L}} {\lVert \tvecx - \vecx' \rVert}_2$.

    \item[BADGE~\cite{badge2019ash}] Select one point at a time by the $k$-means$^{++}$ initialization algorithm, based on gradient embeddings w.r.t.\ the weights of the last layer.
    \item[Typiclust~\cite{typiclust2022hacohen}]  Run a clustering algorithm, \eg $k$-means. For each cluster, select a point with the highest ``typicality'' using $m$-Nearest Neighbors:
    $\mathrm{typicality}(\vecx) \defeq {\left(\frac{1}{m} \sum_{\vecx' \in m\text{-NN}(\vecx)} {\lVert \vecx - \vecx' \rVert}_2 \right)}^{-1}.$
    \item[ProbCover~\cite{probcover2022yehuda}] Construct a graph where nodes represent data points in $\mathcal{U}$ and there is an edge between $\vecx$ and $\vecx'$ if $\lVert \vecx - \vecx' \rVert \le \delta$. $\delta$ is pre-defined based on heuristics proposed by the authors (refer to \cref{app:sec:delta} for more details).
    Select points one at a time with the highest number of edges.
    \item[KernelHerding~\cite{kherding2010chen}] Select one point at a time based on \eqref{eq:kherding_approx}.
\end{description}

\subsubsection{Evaluation metric}
Following \cite{typiclust2022hacohen,probcover2022yehuda}, we select $C$ data points for annotation at each iteration, where $C$ is the number of classes.
We report the mean and standard deviation of test accuracies for $5$ runs at each of $10$ iterations, except for ImageNet where we conduct $3$ runs at each of $5$ iterations.

\subsubsection{Implementation}
We consider classification both with a 1-NN classifier,
where $\hat f(\vecx)$ is the label of the closest point to $\vecx$ in $\mathcal L$,
and with a linear classifier trained on $\mathcal L$ with cross-entropy loss (multiclass logistic regression).
All results using a linear classifier are presented in \cref{app:sec:exp_linear}.\footnote{BADGE~\cite{badge2019ash} is presented only for linear classifiers, since the gradient embeddings are not defined for 1-NN classifiers.}

Our implementation for active learning is based on \href{https://github.com/avihu111/TypiClust}{the code provided by the authors} of \cite{typiclust2022hacohen,probcover2022yehuda}.
Since the code for the heuristic choice of $\delta$ is not provided,
we implement it as described in \cref{app:sec:delta}.
For self-supervised learning features, we employ SimCLR~\cite{simclr2020chen} and SCAN~\cite{scan2020van} using \href{https://github.com/wvangansbeke/Unsupervised-Classification}{the code} provided by the authors of \cite{scan2020van}, and DINO\cite{dino2021caron} using \href{https://github.com/facebookresearch/dino}{its public implementation}.
We use a Gaussian kernel for both MaxHerding and kernel $k$-medoids, with lengthscale fixed to $1$ throughout the experiments, but we also provide a comparison of kernel choices in \cref{fig:kernels}.

\subsection{Comparison on Benchmarks with Prior Methods}
\label{subsec:exp:sota}

We compare MaxHerding to the other active learning methods on benchmark datasets: CIFAR10~\cite{cifar10krizhevsky}, CIFAR100~\cite{cifar100Krizhevsky}, TinyImageNet~\cite{tinyimagenet} and ImageNet~\cite{imagenet2009deng}.
We extract features using DINO for ImageNet, and SimCLR for the rest.

\cref{fig:sota_1nn} shows that uncertainty-based methods are even worse than random selection, as previously reported~\cite{typiclust2022hacohen,probcover2022yehuda}.\footnote{In high-budget settings, uncertainty methods outperform other methods; see \cref{fig:budget}.}
Kernel herding is significantly better than uncertainty-based methods, but worse than Typiclust and MaxHerding.
Although ProbCover performs the best on TinyImageNet, it is significantly worse than MaxHerding and Typiclust on the CIFAR datasets.
This demonstrates the instability of the heuristics for choosing its $\delta$ radius parameter.

MaxHerding consistently outperforms other methods including Typiclust, the most competitive method.
MaxHerding especially outperforms Typiclust on ImageNet;
this is likely because of the hardcoded $500$-cluster limit of Typiclust, smaller than the number of ImageNet classes.
Increasing that limit without further attention, however, can cause issues as described in \cref{subsec:method:non_greedy}.

\begin{figure*}[t!]
    \centering
    \includegraphics[width=\linewidth]{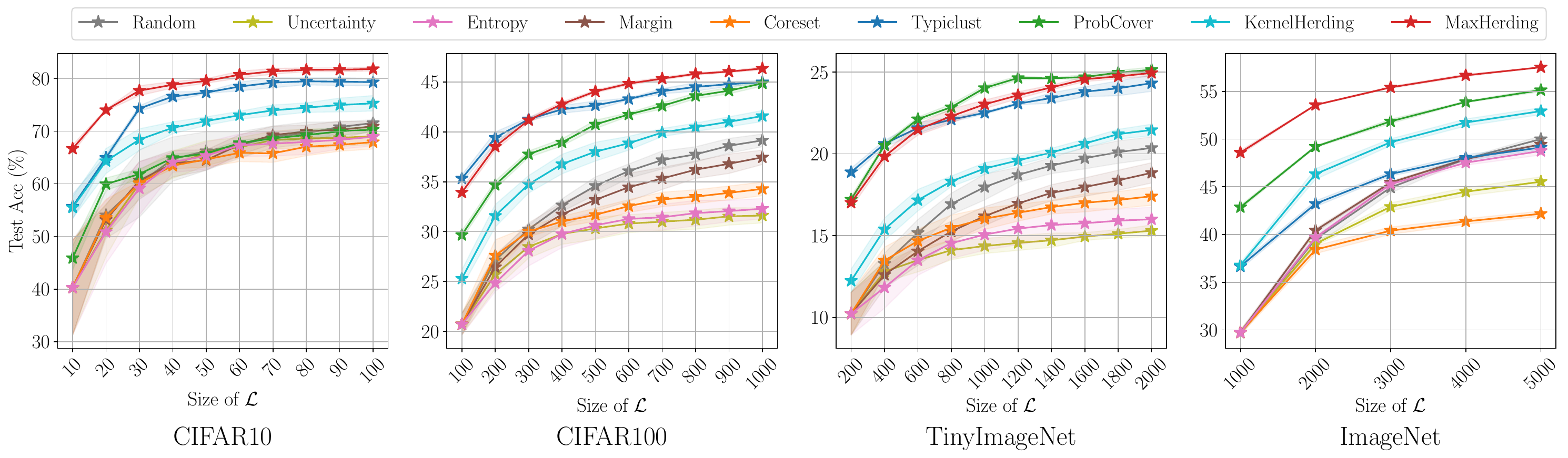}
    \caption{
    Comparison on benchmark datasets using 1-NN classifier.
    }
    \label{fig:sota_1nn}
\end{figure*}

\begin{figure*}[t!]
    \centering
    \includegraphics[width=\linewidth]{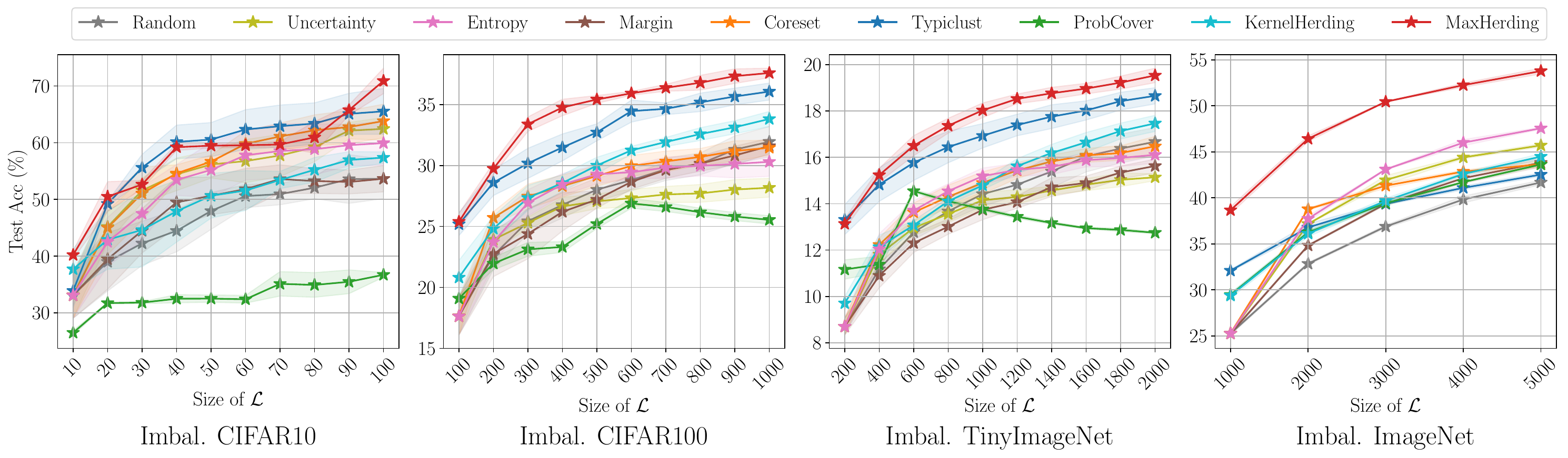}
    \caption{
    Comparison on imbalanced datasets using 1-NN classifier.  
    }
    \label{fig:imbal_1nn}
    \vspace{-3mm}
\end{figure*}

\subsection{Comparison on Imbalanced Datasets}
\label{subsec:exp:imbal}
Although datasets with imbalances in class frequencies are common in practical,
previous works have focused on benchmark datasets with balanced classes.
We generate imbalanced datasets using the benchmark datasets with a standard long-tailed imbalance generation method \cite{imbal2019cui}; details are in \cref{app:sec:exp_linear}.

Low budget active learning on imbalanced datasets is significantly harder than balanced datasets, because selecting data points from the densest region may result in ignoring some classes with a small number of samples.
In extreme cases, test accuracy may even decrease with more labels; this occurs with ProbCover on Imbalanced CIFAR100 and TinyImageNet in \cref{fig:imbal_1nn}.

In this setting, Coreset generally performs better than random selection -- not the case on the balanced datasets.
More importantly, the margin between MaxHerding and Typiclust is larger than on the balanced datasets.
We conjecture that this may be related to the different notions of ``coverage'' used by the two algorithms, with Typiclust more drawn to choose additional points in dense areas (due to its typicality scores) than MaxHerding based on Gaussian kernels.

\subsection{Ablation Study}
\label{subsec:exp:ablation}

\subsubsection{Greedy \textit{vs.} non-greedy} 
Due to the high computational overhead of the non-greedy algorithm 
(as we will see in \cref{fig:runtime}), we did not provide its performance in previous results.
In \cref{fig:greedy_vs_nongreedy}, we compare the greedy (MaxHerding) and non-greedy (kernel $k$-medoids) for the generalized coverage (left) and test accuracy (right) on CIFAR100. 
The greedy and non-greedy algorithms are not significantly different in either coverage or test accuracy, except for the first few iterations.

\subsubsection{Comparison in runtime} 
We compare the runtime of low-budget active learning methods on CIFAR100 and TinyImageNet in \cref{fig:runtime}.
Coreset and ProbCover are extremely fast; Coreset only needs to find the closest labeled points to each candidate $\tvecx$, and ProbCover usually maintains quite a sparse graph with a reasonably large $\delta$.
However, their downstream performance is significantly worse than MaxHerding and Typiclust,
as shown in \cref{subsec:exp:sota,subsec:exp:imbal}.

Typiclust, which usually gives active learning performance comparable to (but worse than) that of MaxHerding, generally has runtime a factor of three or four longer.
Our proposed kernel $k$-medoids generally matches or slightly outperforms MaxHerding, but with more than ten times as much runtime. %

\begin{figure*}[t!]
    \centering
    \begin{subfigure}[b]{0.48\textwidth} %
        \includegraphics[width=\textwidth]{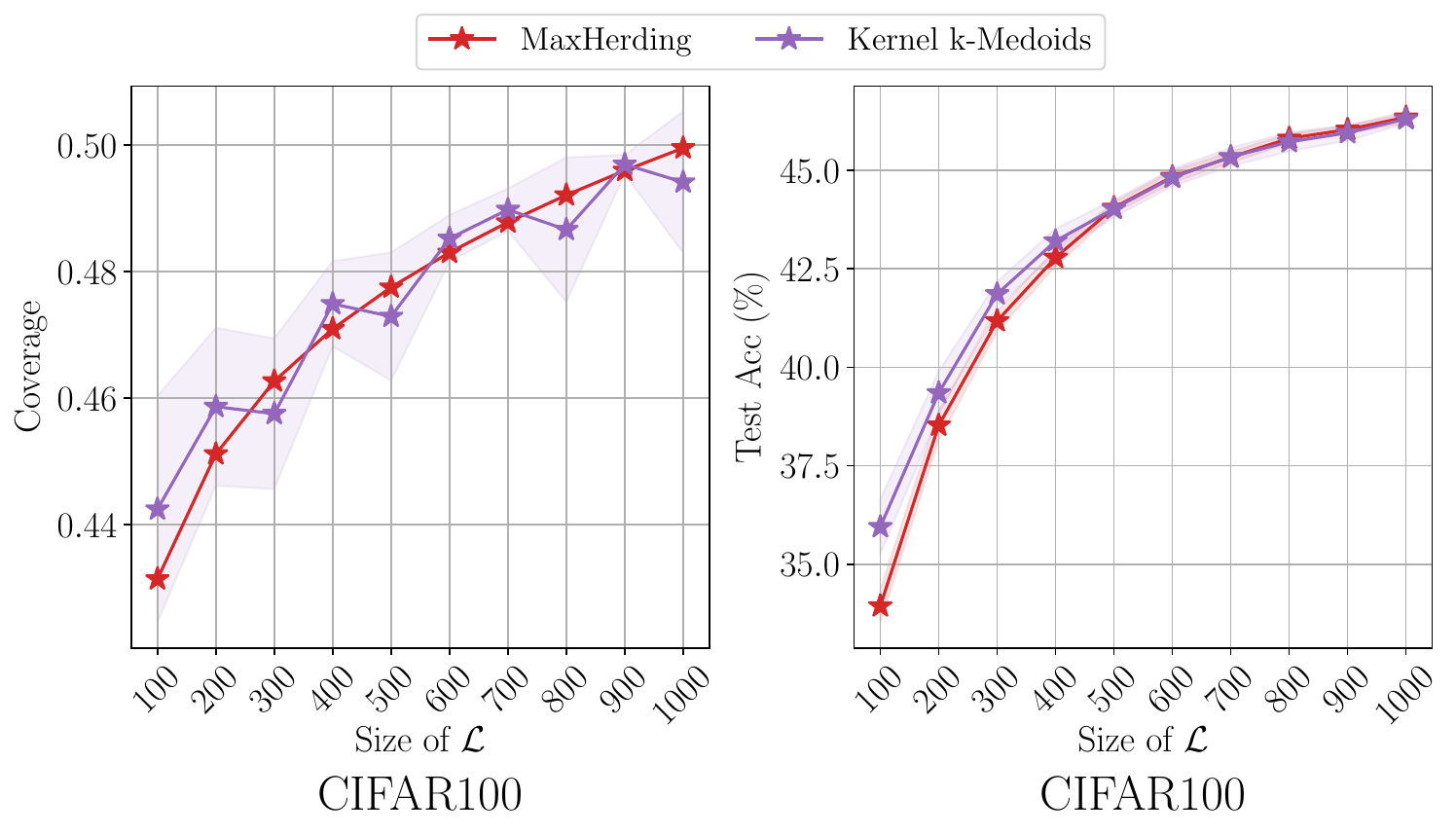}
        \caption{Greedy \textit{vs.} Non-greedy}
        \label{fig:greedy_vs_nongreedy}
    \end{subfigure}
    \hfill
    \begin{subfigure}[b]{0.48\textwidth} %
        \includegraphics[width=\textwidth]{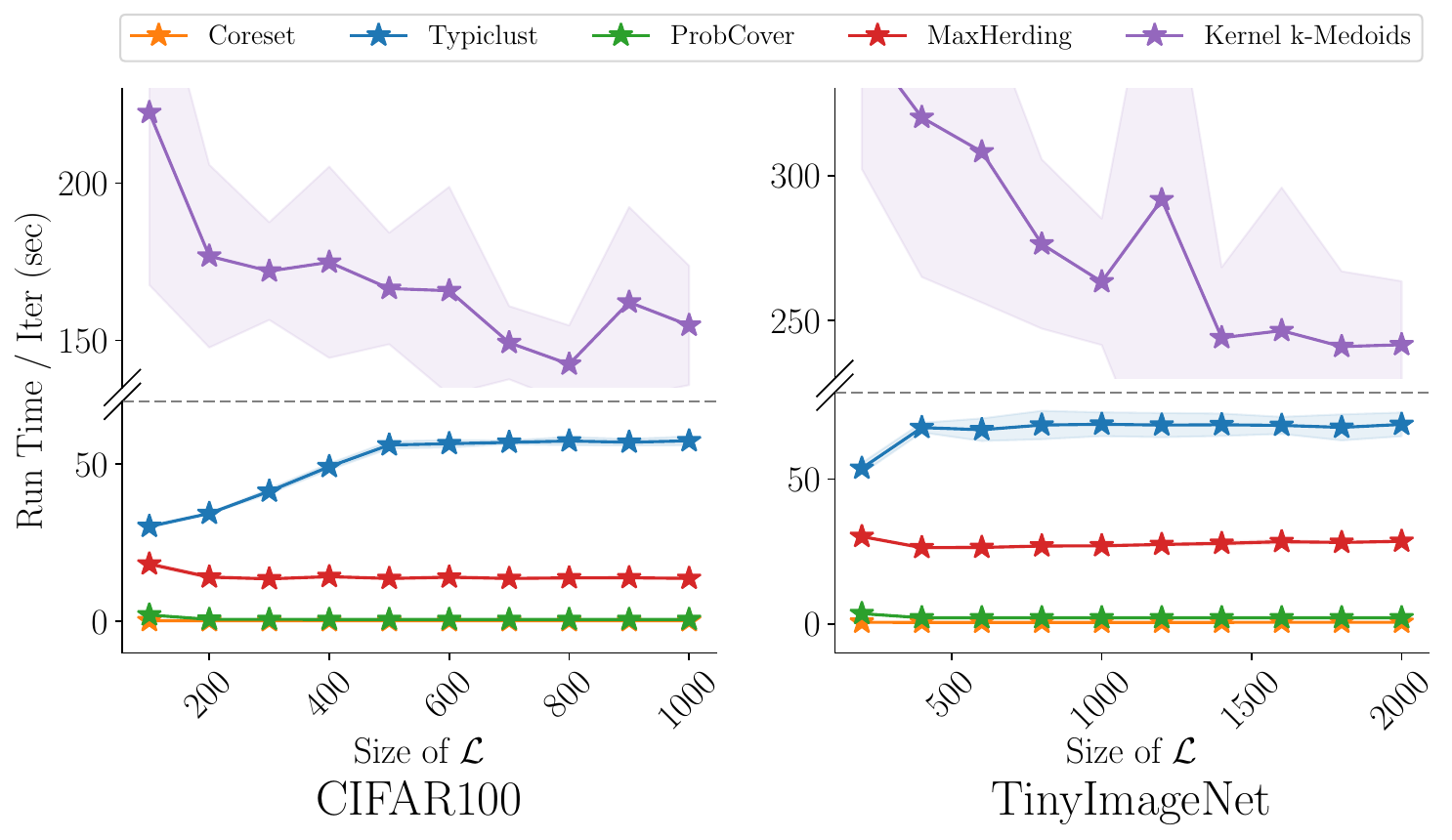}
        \caption{Runtime per selection -- note the axis break}
        \label{fig:runtime}
    \end{subfigure}
    \caption{(a) compares the greedy (MaxHerding) and non-greedy (kernel $k$-medoids) algorithms on CIFAR100, and (b) compares low budget active learning methods in terms of runtime in seconds per selection on CIFAR100 and TinyImageNet.}
\end{figure*}

\subsubsection{Mid- and high-budget regimes} 
Although our main focus is on the low-budget regime, in practice, these regimes are not precisely defined and hard to know in advance.%
\footnote{SelectAL~\cite{select_al2024hacohen} introduced an algorithm to identify the regime, but it requires multiple model re-trainings for cross-validation.}
In \cref{fig:budget}, we compare methods in the mid-budget regime ($100$ samples per iteration) and high-budget regime ($1{,}000$ samples per iteration) on CIFAR10. 
In the mid-budget regime, MaxHerding still performs the best, although Margin catches up at the end.
As expected, in the high-budget regime, Margin and Entropy outperform other methods; MaxHerding still significantly outperforms Typiclust and ProbCover, which shows the robustness of MaxHerding compared to other low-budget active learning methods.

\subsubsection{Robustness to feature embedding} 
We also compare methods using more recent self-supervised learning features: SCAN~\cite{scan2020van} and DINO~\cite{dino2021caron} on CIFAR100.
We obtain SCAN features using a ResNet18 trained on CIFAR100, while for DINO features we use a ViT model pre-trained on ImageNet~\cite{imagenet2009deng}.
As the ViT model is not directly trained on CIFAR100, the performance using the DINO features is generally lower; we still use DINO features because we think this setting reflects a highly practical use of active learning, where time or resources are insufficient to train a self-supervised learning model from scratch.

With these features, we observe similar patterns as we did with previous experiments; uncertainty-based methods are worse than random selection, and MaxHerding generally outperforms the other methods.

\begin{figure*}[t!]
    \centering
    \begin{subfigure}[b]{0.49\textwidth} %
        \includegraphics[width=\textwidth]{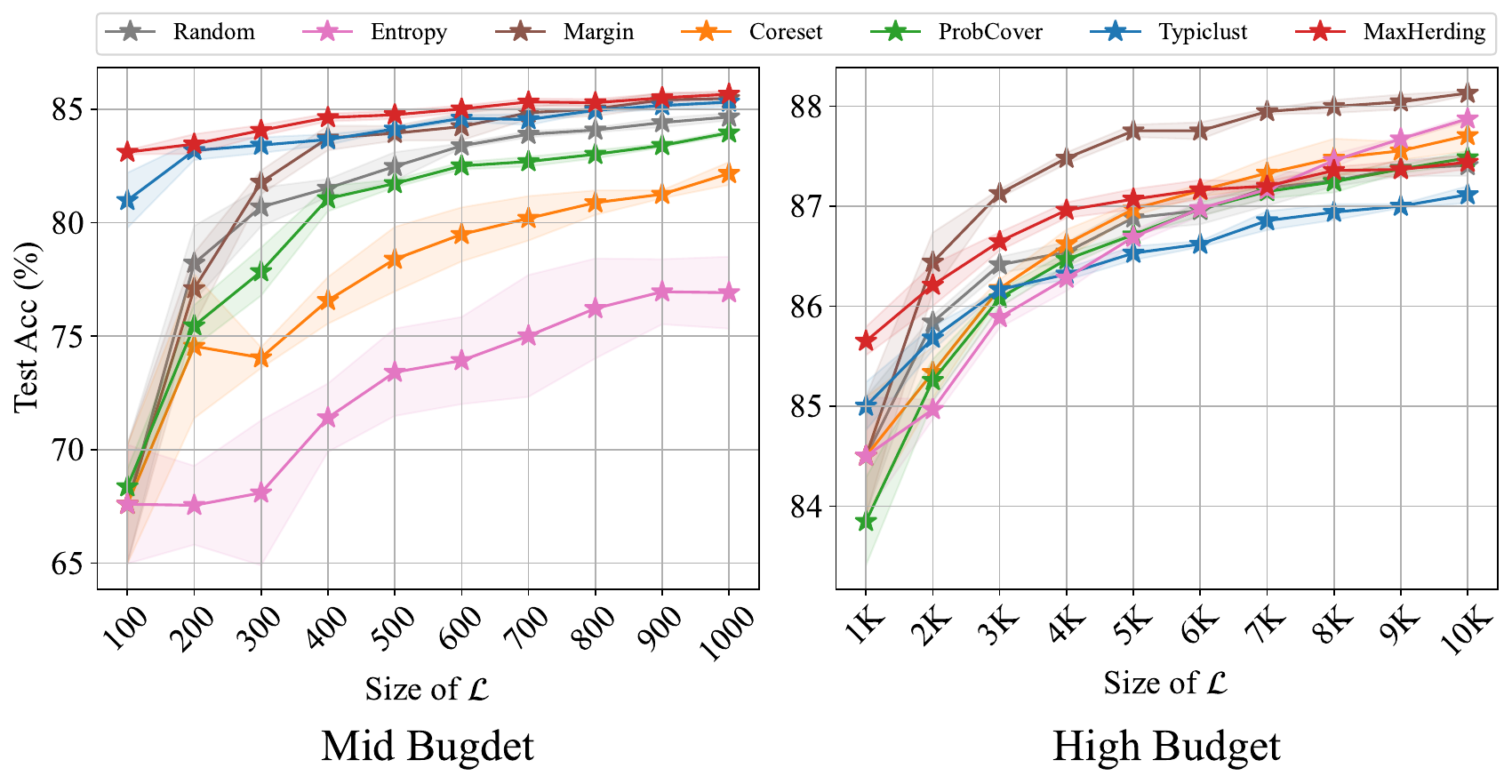}
        \caption{Mid and high budget regime}
        \label{fig:budget}
    \end{subfigure}
    \hfill
    \begin{subfigure}[b]{0.49\textwidth} %
        \includegraphics[width=\textwidth]{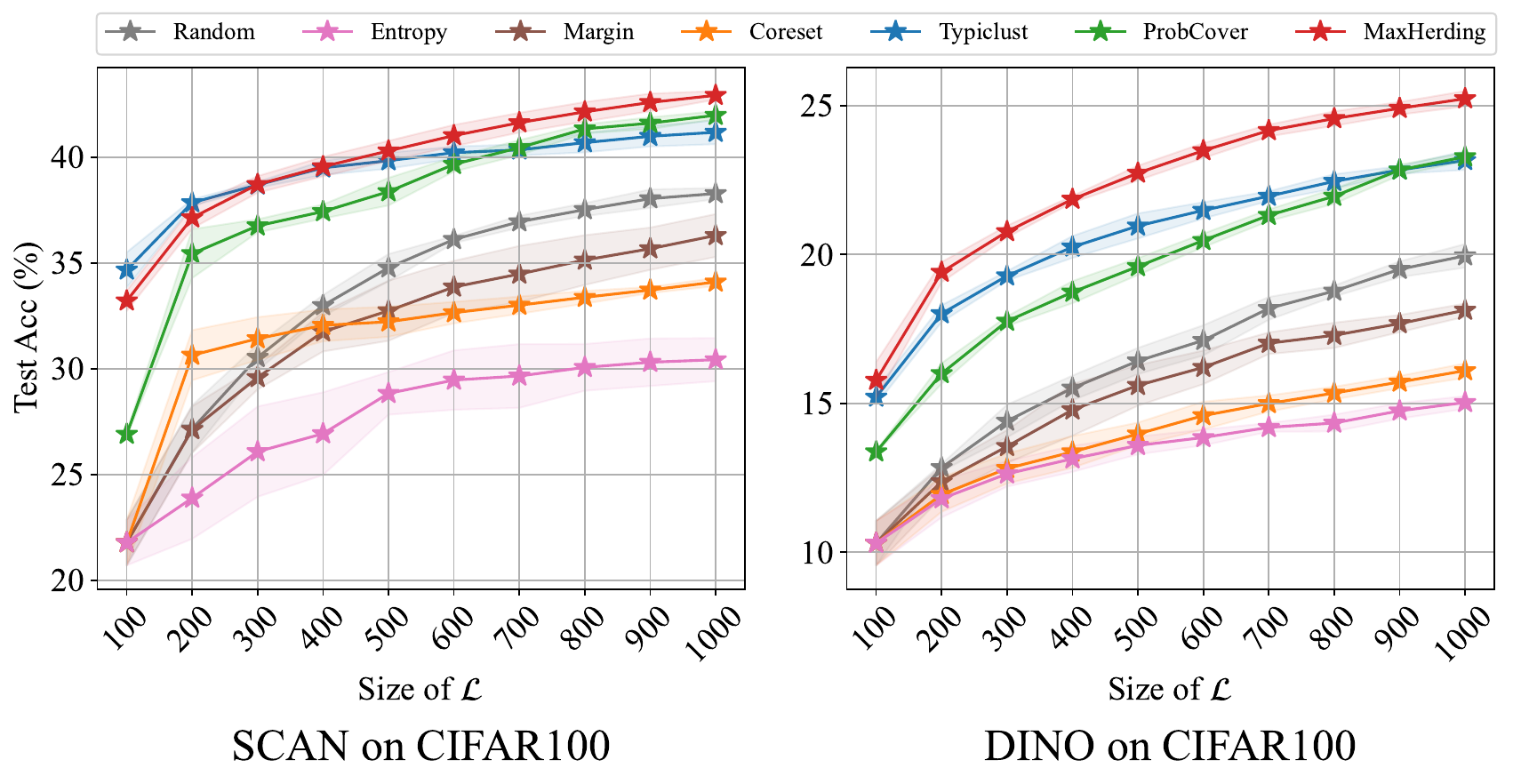}
        \caption{SCAN and DINO features}
        \label{fig:other_features}
    \end{subfigure}
    \caption{(a) compares active learning methods in different budget regimes, and (b) compares them using different self-supervised learning features: SCAN and DINO.}
\end{figure*}

\begin{wrapfigure}[8]{r}{0.38\textwidth}
    \vspace{-8mm}
    \centering
    \includegraphics[width=1.0\linewidth]{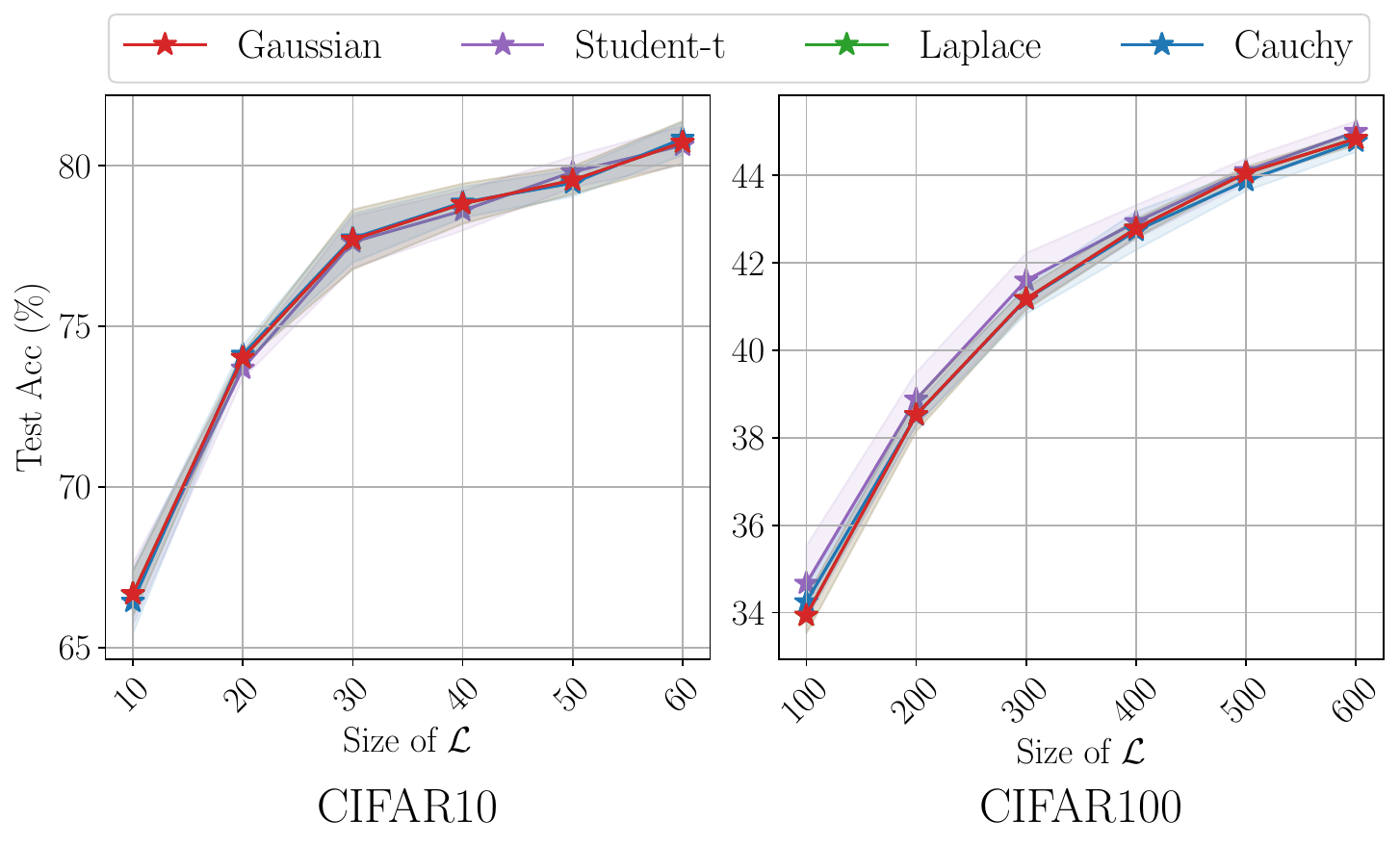} 
    \vspace{-7.5mm}
    \caption{Comparison of kernels.}
    \label{fig:kernels}
\end{wrapfigure}
\subsubsection{Different Kernel Functions} 
We compare MaxHerding with different kernel functions -- Gaussian, Student-t, Laplace and Cauchy -- on CIFAR10 and CIFAR100 datasets.
These kernels are all functions of $\lVert \vecx - \vecx' \rVert/\delta$;
we fix $\delta=1$ for all kernels.
\cref{fig:kernels} shows that all choices of kernels perform similarly; Student-t is very slightly better than others on CIFAR100, but the margin is not significant.
MaxHerding is thus robust to the choice of kernel.

\section{Conclusion}
In this study, we introduced a generalized notion of coverage,
which allows for smoother measurements of whether a point is ``representative'' of nearby data points.
We proposed novel low-budget active learning approaches aimed at maximizing this generalized coverage, employing both greedy and non-greedy algorithms.
We show that our greedy algorithm, MaxHerding, is closely connected to kernel herding as well as including ProbCover as a special case; it is simple and fast, but surpasses the performance of existing active learning methods.

The non-greedy algorithm, a version of kernel $k$-medoids, performs slightly better than MaxHerding, but its computation demands are significantly higher.
Therefore, we recommended practitioners prioritize MaxHerding over the non-greedy algorithm kernel $k$-medoids algorithm.

Recently, active learning methods designed for fine-tuning~\cite{activeft2023xie} have gained popularity. It would be an interesting future work to demonstrate how low-budget active learning approaches perform in the fine-tuning of large models.

\section*{Acknowledgements}
This work was enabled in part by support provided by the
Natural Sciences and Engineering Research Council of Canada,
the Canada CIFAR AI Chairs program,
Mitacs through the Mitacs Accelerate program,
Advanced Research Computing at the University of British Columbia,
the BC DRI Group,
and the Digital Research Alliance of Canada.
Junhyug Noh was supported by Institute of Information \& communications Technology Planning \& Evaluation (IITP) grant funded by the Korea government (MSIT) (No.RS-2022-00155966).

\bibliographystyle{splncs04}
\bibliography{main}

\begin{thebibliography}{10}
\providecommand{\url}[1]{\texttt{#1}}
\providecommand{\urlprefix}{URL }
\providecommand{\doi}[1]{https://doi.org/#1}

\bibitem{kmedoids2016aghaee}
Aghaee, A., Ghadiri, M., Baghshah, M.S.: Active distance-based clustering using k-medoids. In: PAKDD (2016)

\bibitem{hybrid_active_meta2021al}
Al-Shedivat, M., Li, L., Xing, E., Talwalkar, A.: On data efficiency of meta-learning. In: AISTAT (2021)

\bibitem{badge2019ash}
Ash, J.T., Zhang, C., Krishnamurthy, A., Langford, J., Agarwal, A.: Deep batch active learning by diverse, uncertain gradient lower bounds. ICLR  (2020)

\bibitem{herding_cond2012bach}
Bach, F., Lacoste-Julien, S., Obozinski, G.: On the equivalence between herding and conditional gradient algorithms. ICML  (2012)

\bibitem{active_meta2023bae}
Bae, W., Wang, J., Sutherland, D.J.: Exploring active learning in meta-learning: Enhancing context set labeling. arXiv preprint arXiv:2311.02879  (2023)

\bibitem{berlinet-thomas-agnan}
Berlinet, A., Thomas-Agnan, C.: Reproducing Kernel Hilbert Spaces in Probability and Statistics. Springer New York (2003)

\bibitem{dino2021caron}
Caron, M., Touvron, H., Misra, I., J{\'e}gou, H., Mairal, J., Bojanowski, P., Joulin, A.: Emerging properties in self-supervised vision transformers. In: ICCV (2021)

\bibitem{simclr2020chen}
Chen, T., Kornblith, S., Norouzi, M., Hinton, G.: A simple framework for contrastive learning of visual representations. In: ICML (2020)

\bibitem{steinpoints2018chen}
Chen, W.Y., Mackey, L., Gorham, J., Briol, F.X., Oates, C.: Stein points. In: ICML (2018)

\bibitem{kherding2010chen}
{Chen, Yutian and Welling, Max and Smola, Alex}: Super-samples from kernel herding. In: UAI (2010)

\bibitem{steinwart-christmann}
Christmann, A., Steinwart, I.: Support Vector Machines. Springer New York (2008)

\bibitem{imbal2019cui}
Cui, Y., Jia, M., Lin, T.Y., Song, Y., Belongie, S.: Class-balanced loss based on effective number of samples. In: CVPR (2019)

\bibitem{imagenet2009deng}
Deng, J., Dong, W., Socher, R., Li, L.J., Li, K., Fei-Fei, L.: {ImageNet}: A large-scale hierarchical image database. In: CVPR (2009)

\bibitem{voc2015Everingham}
Everingham, M., Eslami, S.M.A., Van~Gool, L., Williams, C.K.I., Winn, J., Zisserman, A.: The pascal visual object classes challenge: A retrospective. IJCV  (2015)

\bibitem{frank1956algorithm}
Frank, M., Wolfe, P., et~al.: An algorithm for quadratic programming. Naval research logistics quarterly  \textbf{3}(1-2),  95--110 (1956)

\bibitem{emoc2014frey}
Freytag, A., Rodner, E., Denzler, J.: Selecting influential examples: Active learning with expected model output changes. In: ECCV (2014)

\bibitem{plant_location1974frieze}
Frieze, A.M.: A cost function property for plant location problems. Mathematical Programming pp. 245--248 (1974)

\bibitem{eer_mi2007guo}
Guo, Y., Greiner, R.: Optimistic active-learning using mutual information. In: IJCAI. pp. 823--829 (2007)

\bibitem{typiclust2022hacohen}
Hacohen, G., Dekel, A., Weinshall, D.: Active learning on a budget: Opposite strategies suit high and low budgets. In: ICML (2022)

\bibitem{select_al2024hacohen}
Hacohen, G., Weinshall, D.: How to select which active learning strategy is best suited for your specific problem and budget. NeurIPS  (2024)

\bibitem{wherding2012huszar}
{Husz{\'a}r, Ferenc and Duvenaud, David}: Optimally-weighted herding is bayesian quadrature. In: UAI (2012)

\bibitem{emoc2016kading}
K{\"a}ding, C., Rodner, E., Freytag, A., Denzler, J.: Active and continuous exploration with deep neural networks and expected model output changes. In: NIPSW (2016)

\bibitem{emoc_reg2018kading}
K{\"a}ding, C., Rodner, E., Freytag, A., Mothes, O., Barz, B., Denzler, J., AG, C.Z.: Active learning for regression tasks with expected model output changes. In: BMVC (2018)

\bibitem{pam2009kaufman}
Kaufman, L., Rousseeuw, P.J.: Finding groups in data: an introduction to cluster analysis. John Wiley \& Sons (2009)

\bibitem{submodular2014krause}
Krause, A., Golovin, D.: Submodular function maximization. Tractability p.~3 (2014)

\bibitem{cifar10krizhevsky}
Krizhevsky, A.: Learning multiple layers of features from tiny images (2009)

\bibitem{cifar100Krizhevsky}
Krizhevsky, A., Nair, V., Hinton, G.: Cifar-100 (canadian institute for advanced research), \url{http://www.cs.toronto.edu/~kriz/cifar.html}

\bibitem{dpp2012kulesza}
Kulesza, A., Taskar, B., et~al.: Determinantal point processes for machine learning. Foundations and Trends{\textregistered} in Machine Learning  \textbf{5}(2--3),  123--286 (2012)

\bibitem{cond_grad1966levitin}
Levitin, E.S., Polyak, B.T.: Constrained minimization methods. USSR Computational mathematics and mathematical physics  \textbf{6}(5),  1--50 (1966)

\bibitem{heterogeneous1994lewis}
Lewis, D.D., Catlett, J.: Heterogeneous uncertainty sampling for supervised learning. In: Machine learning proceedings 1994. pp. 148--156 (1994)

\bibitem{sequential1994lewis}
Lewis, D.D., Gale, W.A.: A sequential algorithm for training text classifiers. In: SIGIR. pp. 3--12 (1994)

\bibitem{low2022mahmood}
Mahmood, R., Fidler, S., Law, M.T.: Low budget active learning via wasserstein distance: An integer programming approach. In: ICLR (2022)

\bibitem{rust2014matsakis}
Matsakis, N.D., Klock, F.S.: The {R}ust language. ACM SIGAda Ada Letters  (2014)

\bibitem{facility_location1990mirchandani}
Mirchandani, P.B., Francis, R.L.: Discrete location theory (1990)

\bibitem{tinyimagenet}
mnmoustafa, M.A.: Tiny imagenet (2017), \url{https://kaggle.com/competitions/tiny-imagenet}

\bibitem{lookahead2022mohamadi}
Mohamadi, M.A., Bae, W., Sutherland, D.J.: Making look-ahead active learning strategies feasible with neural tangent kernels. In: NeurIPS (2022)

\bibitem{Muandet_2017}
Muandet, K., Fukumizu, K., Sriperumbudur, B., Schölkopf, B.: Kernel mean embedding of distributions: A review and beyond. Foundations and Trends® in Machine Learning p. 1–141 (2017)

\bibitem{submodular_greedy1978nemhauser}
Nemhauser, G.L., Wolsey, L.A., Fisher, M.L.: An analysis of approximations for maximizing submodular set functions—i. Mathematical programming  (1978)

\bibitem{alfa_mix2022parvaneh}
Parvaneh, A., Abbasnejad, E., Teney, D., Haffari, G.R., Van Den~Hengel, A., Shi, J.Q.: Active learning by feature mixing. In: CVPR (2022)

\bibitem{bq2003rasmussen}
Rasmussen, C.E., Ghahramani, Z.: Bayesian monte carlo (2003)

\bibitem{eer2001roy}
Roy, N., McCallum, A.: Toward optimal active learning through monte carlo estimation of error reduction. In: ICML (2001)

\bibitem{margin2001scheffer}
Scheffer, T., Decomain, C., Wrobel, S.: Active hidden markov models for information extraction. In: ISIDA (2001)

\bibitem{fasterPAM2022schubert}
Schubert, E., Lenssen, L.: Fast k-medoids clustering in {R}ust and {P}ython. Journal of Open Source Software p.~4183 (2022)

\bibitem{fastPAM2019schubert}
Schubert, E., Rousseeuw, P.J.: Faster k-medoids clustering: improving the {PAM}, {CLARA}, and {CLARANS} algorithms. In: Similarity Search and Applications: 12th International Conference, SISAP 2019, Newark, NJ, USA, October 2--4, 2019, Proceedings 12. pp. 171--187 (2019)

\bibitem{fasterPAM2021schubert}
Schubert, E., Rousseeuw, P.J.: Fast and eager k-medoids clustering: O(k) runtime improvement of the {PAM}, {CLARA}, and {CLARANS} algorithms. Information Systems  (2021)

\bibitem{coreset2017sener}
Sener, O., Savarese, S.: Active learning for convolutional neural networks: A core-set approach. In: ICLR (2018)

\bibitem{al2009settles}
Settles, B.: Active learning literature survey. University of Wisconsin-Madison Department of Computer Sciences (2009)

\bibitem{pixelpick2021Shin}
Shin, G., Xie, W., Albanie, S.: All you need are a few pixels: Semantic segmentation with pixelpick. In: ICCV Workshops (2021)

\bibitem{soudry2018implicit}
Soudry, D., Hoffer, E., Nacson, M.S., Gunasekar, S., Srebro, N.: The implicit bias of gradient descent on separable data. JMLR  \textbf{19} (2018)

\bibitem{scan2020van}
Van~Gansbeke, W., Vandenhende, S., Georgoulis, S., Proesmans, M., Van~Gool, L.: Scan: Learning to classify images without labels. In: ECCV (2020)

\bibitem{kmedian2012voevodski}
Voevodski, K., Balcan, M.F., R{\"o}glin, H., Teng, S.H., Xia, Y.: Active clustering of biological sequences. In: JMLR (2012)

\bibitem{entropy2014wang}
Wang, D., Shang, Y.: A new active labeling method for deep learning. In: IJCNN (2014)

\bibitem{activeft2023xie}
Xie, Y., Lu, H., Yan, J., Yang, X., Tomizuka, M., Zhan, W.: Active finetuning: Exploiting annotation budget in the pretraining-finetuning paradigm. In: CVPR (2023)

\bibitem{kmeans2003xu}
Xu, Z., Yu, K., Tresp, V., Xu, X., Wang, J.: Representative sampling for text classification using support vector machines. In: ECIR (2003)

\bibitem{probcover2022yehuda}
Yehuda, O., Dekel, A., Hacohen, G., Weinshall, D.: Active learning through a covering lens. In: NeurIPS (2022)

\bibitem{eer_gf2003zhu}
Zhu, X., Lafferty, J., Ghahramani, Z.: Combining active learning and semi-supervised learning using gaussian fields and harmonic functions. In: ICMLW. vol.~3 (2003)

\end{thebibliography}

\clearpage
\appendix

\section{Proofs}

\subsection{Proof for \cref{thm:ourbound}}
\label{app:subsec:ourbound}

\ourbound*

\begin{proof}
    We begin with
    \begin{align*}
        &\Pr_{\vecx}(f(\vecx) \ne \hat f(\vecx)) = \E \left[ \mathbbm{1}[f(\vecx) \neq \hat{f}(\vecx) ] \right] \\ 
        &= \E \left[ \mathbbm{1}[f(\vecx) \neq \hat{f}(\vecx) ] (1 - \max_{\vecx' \in \mathcal{L}} k(\vecx, \vecx') ) \right] + \E \left[ \mathbbm{1}[f(\vecx) \neq \hat{f}(\vecx) ]  \cdot \max_{\vecx' \in \mathcal{L}} k(\vecx, \vecx') \right] \\
        &\leq 1 - \E \left[ \max_{\vecx' \in \mathcal{L}} k(\vecx, \vecx') \right]
        + \E \left[ \mathbbm{1}[f(\vecx) \neq \hat{f}(\vecx) ] \cdot \max_{\vecx' \in \mathcal{L}} k(\vecx, \vecx') \right]
    ,\end{align*}
    using in the last line that $\mathbbm{1}[f(\vecx) \ne \hat f(\vecx)] \le 1$.
    Now consider the value
    \[
        g(\vecx) = \mathbbm{1}[f(\vecx) \neq \hat{f}(\vecx) ] \cdot \max_{\vecx' \in \mathcal{L}} k(\vecx, \vecx')
    .\]
    Recall that $\hat f(\vecx)$ is $f(\vecx')$, where $\vecx' = \argmin_{\tvecx \in \mathcal L} \lVert \vecx - \tvecx \rVert$.
    Thus, for any given $\vecx$, there are two cases:
    \begin{enumerate}
        \item $f(\vecx) = \hat f(\vecx)$. In this case, $g(\vecx) = 0$, and is hence upper-bounded by \[ \max_{\vecx' : f(\vecx) \ne f(\vecx')} k(\vecx, \vecx') \ge 0 .\]
        \item $f(\vecx) \ne \hat f(\vecx)$.
            Let $\tvecx$ be the nearest point in $\mathcal L$ to $\vecx$.
            We know that $f(\vecx) \ne f(\tvecx)$,
            and hence
            \[
                k(\vecx, \tvecx)
                \le \max_{\vecx' : f(\vecx) \ne f(\vecx')} k(\vecx, \vecx')
            .\]
            We also know that because $\tvecx$ is the closest point in $\mathcal L$ to $\vecx$,
            \[
                k(\vecx, \tvecx)
                = \max_{\vecx' \in \mathcal L} k(\vecx, \vecx')
            .\]
            Thus
            \[
                \max_{\vecx' \in \mathcal L} k(\vecx, \vecx')
                \le \max_{\vecx' : f(\vecx) \ne f(\vecx')} k(\vecx, \vecx')
            ,\]
            as desired.
    \end{enumerate}
    Thus $\E g(\vecx) \le \E \max_{\vecx' : f(\vecx) \ne f(\vecx')} k(\vecx, \vecx')$,
    giving the desired result.
    \qed
\end{proof}

\subsection{Proof for \cref{thm:bound}}
\label{app:subsec:bound}

\bound*

\begin{proof}
    When $k$ is a top-hat function, the first term of \eqref{eq:ourbound} becomes
    \[
        1 - \E \left[ \max_{\vecx' \in \mathcal{L}} k(\vecx, \vecx') \right]
        = 1- \E \Big[ \max_{\vecx' \in \mathcal L} \mathbbm{1}\left[ \lVert \vecx - \vecx' \rVert \le \delta \right] \Big]
    ,\]
    agreeing with the first term of \eqref{eq:probcover_bound}.
    For the other term, we have that
    \begin{multline*}
        \E \max_{\vecx' : f(\vecx) \ne f(\vecx')} k(\vecx, \vecx')
        = \E \max_{\vecx' : f(\vecx) \ne f(\vecx')} \mathbbm{1}[ \lVert \vecx - \vecx' \rVert \le \delta ]
    \\
        = \Pr\left( \exists \vecx' \text{ s.t. } f(\vecx) \ne f(\vecx') \text{ and } \lVert \vecx - \vecx' \rVert \le \delta \right)
    ,\end{multline*}
    which is one minus the probability that all points within a radius $\delta$ of $\vecx$ have the same label.
    Thus, \eqref{eq:probcover_bound} is a special case of \eqref{eq:ourbound}.
    \qed
\end{proof}

\subsection{Proof for \cref{thm:maxherding}}
\label{app:subsec:maxherding}

\maxherding*

\begin{proof}
    We will show the second term in \cref{eq:maxherding-reframed} is equal to $0$, and the first term is equivalent to \cref{eq:max_herding}. For $\forall \, \mathbf{x} \in \mathcal{L}, \vecx' \neq \vecx \in \mathcal{X}$, since $k(\vecx, \vecx) \geq k(\vecx, \vecx')$,
    \begin{align}
        k(\mathbf{x}, \tvec{x}; \mathcal{L}) &= \max(k(\mathbf{x}, \tvec{x}) - \max_{\mathbf{x}' \in \mathcal{L}} k(\mathbf{x}, \mathbf{x}'), 0) \\
        &= \max(k(\mathbf{x}, \tvec{x}) - k(\mathbf{x}, \mathbf{x}), 0) = 0.
    \end{align}
    Hence, the second term becomes $0$.
    Now, we convert the max kernel function into the following form and plug it into the first term,
    \begin{align}
        k(\mathbf{x}, \tvecx; \mathcal{L}) = \max_{\vecx' \in \mathcal{L} \cup \{ \tvec{x} \}} k(\vecx, \vecx') - \max_{\vecx' \in \mathcal{L} } k(\vecx, \vecx').
    \end{align} 
    Then, the first term becomes \cref{eq:max_herding}. \qed
\end{proof}

\begin{figure*}[t!]
    \centering
    \begin{subfigure}[b]{0.24\textwidth} %
        \includegraphics[width=\textwidth]{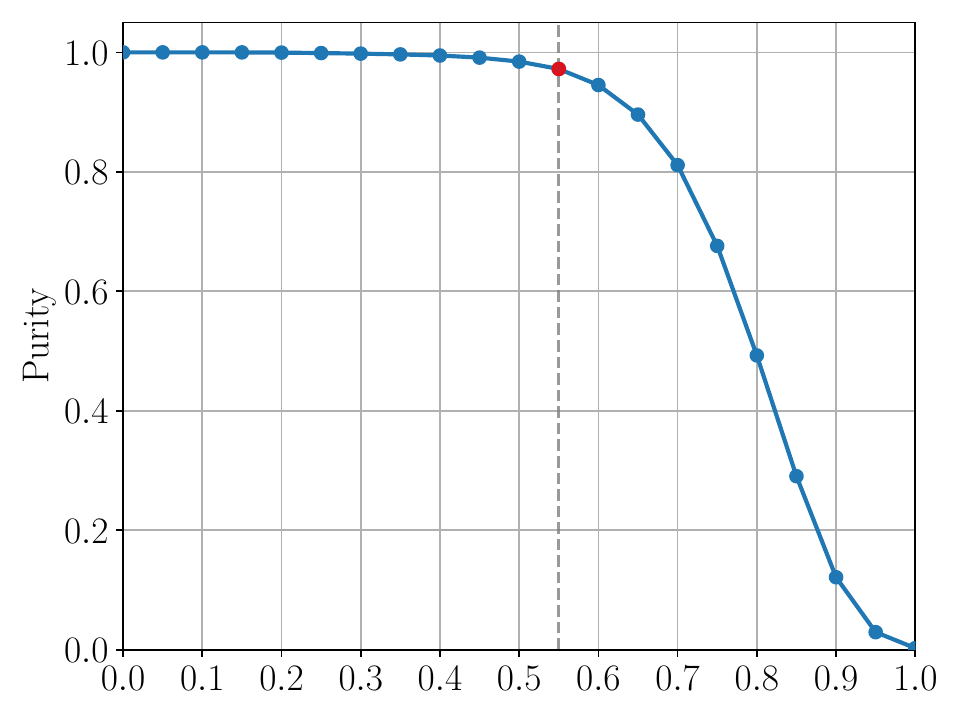}
        \caption{CIFAR10 }
        \label{app:fig:threshold_cifar10}
    \end{subfigure}
    \hfill
    \begin{subfigure}[b]{0.24\textwidth} %
        \includegraphics[width=\textwidth]{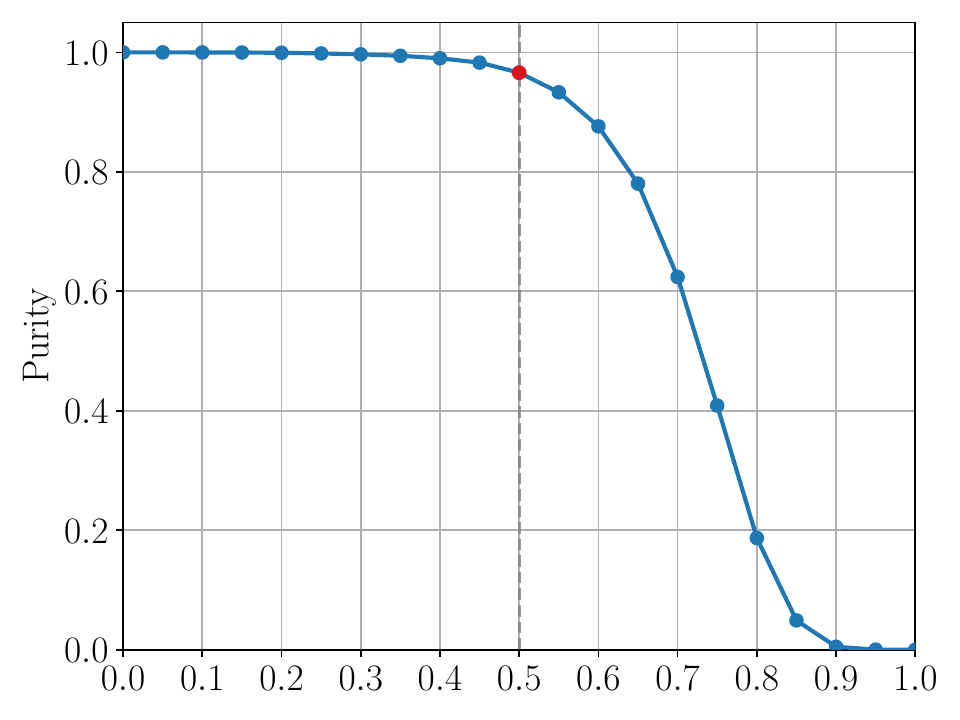}
        \caption{CIFAR100}
        \label{app:fig:threshold_cifar100}
    \end{subfigure}
    \hfill
    \begin{subfigure}[b]{0.24\textwidth} %
        \includegraphics[width=\textwidth]{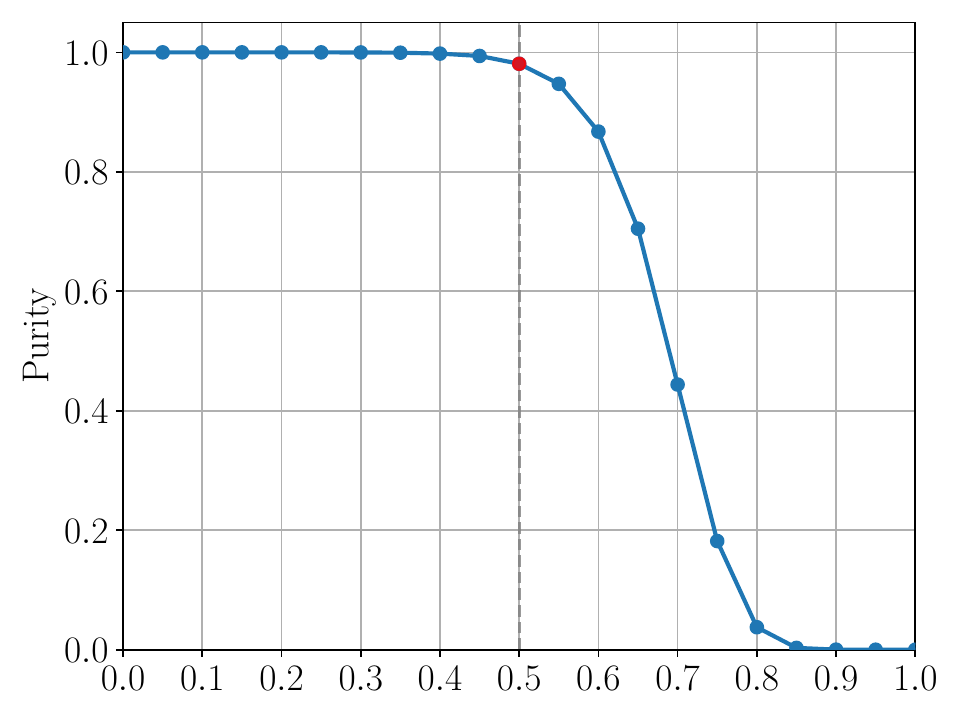}
        \caption{TinyImageNet }
        \label{app:fig:threshold_tinyimagenet}
    \end{subfigure}
    \hfill
    \begin{subfigure}[b]{0.24\textwidth} %
        \includegraphics[width=\textwidth]{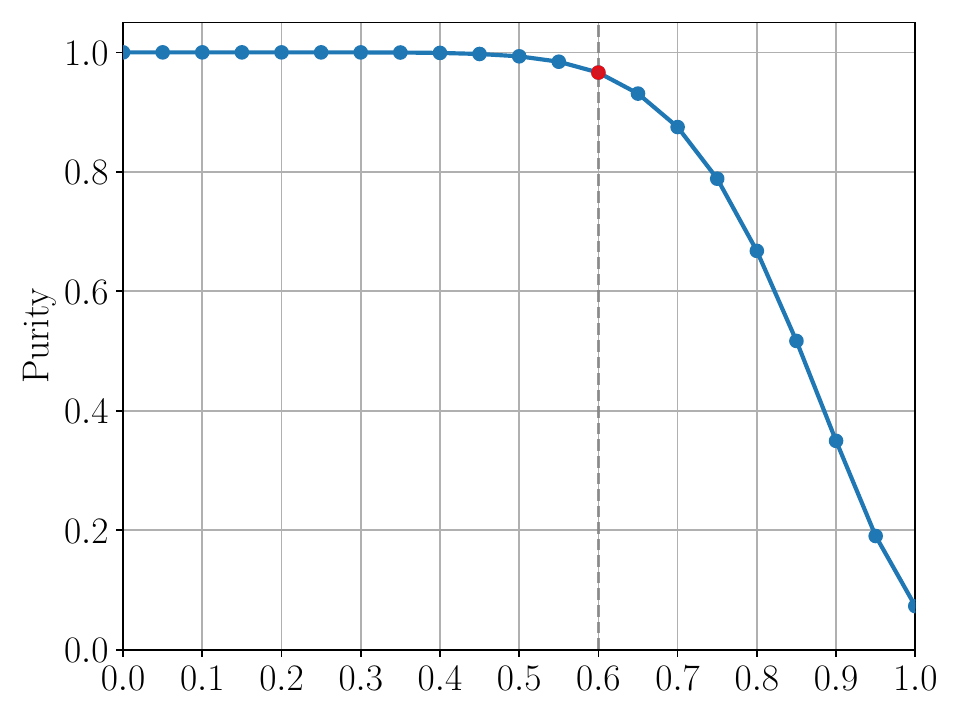}
        \caption{ImageNet}
        \label{app:fig:threshold_imagenet}
    \end{subfigure}
    \caption{Selection of $\delta$-radius using the heuristics suggested by \cite{probcover2022yehuda}.}
    \label{app:fig:delta}
\end{figure*}

\section{Determination of $\delta$ in ProbCover}
\label{app:sec:delta}

Yehuda~\etal\cite{probcover2022yehuda} suggest a heuristic to determine the radius parameter $\delta$ for ProbCover as follows:
\begin{enumerate}
    \item For given self-supervised learning features, \eg SimCLR, it obtains a pseudo-label for each feature using $k$-Means.
    \item It computes the purity of a blob centered at each feature, then, counts how many blobs are ``pure'' (a blob is pure if no features inside the blob are assigned to different clusters).
    \item It finds a radius that gives around $95\%$ of overall purity.
\end{enumerate}

\begin{python}[caption={Python code for selecting $\delta$ based on \cite{probcover2022yehuda}}, label={lst:python_code}]
import numpy as np
import torch
from fast_pytorch_kmeans import KMeans

def compute_dists(features):
    dists = torch.cdist(features, features, p=2)
    return dists

def blob_purity(dists, center_idx, radius, assignments):
    label = assignments[center_idx]
    dist = dists[center_idx]
    total_count = torch.sum(dist < radius)
    cond = torch.logical_and(
        dist < radius, assignments == label)
    pure_count = torch.sum(cond)
    
    if total_count == pure_count:
        return 1
    return 0

def compute_purity(feature_path, num_classes, ratio=1.0):
    # Load self-supervised learning features
    with open(feature_path, 'rb') as f:
        features = np.load(f)
    features = torch.from_numpy(features)

    # Normalize features
    features= torch.nn.functional.normalize(features, dim=1)

    # Find pseudo labels using k-Means
    kmeans = KMeans(n_clusters=num_classes, mode='euclidean')
    assignments = kmeans.fit_predict(features)

    # Compute a distance matrix
    dist_matrix = compute_dists(features)

    best_purity_radius, is_first, purity_rates = 0, True, []
    radiuses = np.linspace(0.05, 1.0, 20)
    for r, radius in enumerate(radiuses):
        # Compute purity for a given radius
        purity_count = 0
        for i in range(num_samples):
            purity_count += blob_purity(
                dist_matrix, i, radius, assignments)
        purity_rate = purity_count / num_samples
        purity_rates.append(purity_rate)

        # If purity < 0.95, determine the "optimal" radius
        if is_first and purity_rate < 0.95:
            best_purity_radius = radiuses[r-1]
            is_first = False
    return best_purity_radius
\end{python}

Despite its importance for the algorithm given that the performance is sensitive to $\delta$, its implementation has not been released.
Hence, we provide our implementation in \cref{lst:python_code}.\footnote{One of the authors of ProbCover have confirmed that our implementation is correct. But we have not been able to obtain their implementation.} 
Also, based on our implementation, we compute the ``optimal'' $\delta$ for each dataset as provided in \cref{app:fig:delta}.

\section{Experiment Results with Linear Classifiers}
\label{app:sec:exp_linear}

We provide the results with a linear classifier on the benchmark datasets in \cref{app:subsec:sota} and imbalanced datasets in \cref{app:subsec:imbal}.
As with \cref{subsec:exp:sota,subsec:exp:imbal}, we report the mean and standard deviation of test accuracies for 5 runs at each of 10 iterations, except for ImageNet where we conduct 3 runs at each of 5 iterations.

\subsection{Comparison with State-of-the-Arts}
\label{app:subsec:sota}

The performance of the linear classifier is slightly worse than 1-NN on CIFAR100 and ImageNet but better on TinyImageNet.

We can observe a similar pattern that is shown in the 1-Nearest Neighbor (NN) classifier. 
Uncertainty-based methods generally underperform compared to Random whereas representation-based methods perform significantly better than Random.
Notably, Typiclust and MaxHerding significantly outperform the other active learning methods, with MaxHerding emerging as the top performer across the datasets.

The disparity between MaxHerding and Typiclust, the second-best model overall, shows less pronounced differences on CIFAR100 and TinyImageNet. However, this distinction becomes more significant on CIFAR10 and ImageNet datasets.
Typiclust is particularly worse than MaxHerding on ImageNet, potentially because of its strict limitation on the number of clusters.

\begin{figure*}[t!]
    \centering
    \includegraphics[width=\linewidth]{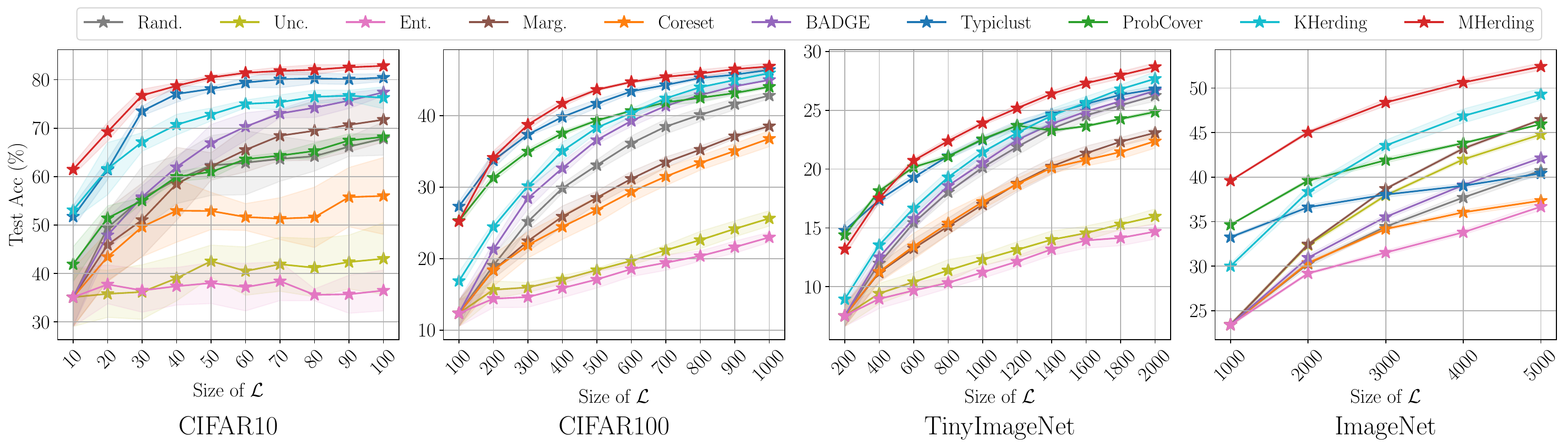}
    \caption{
    Comparison on benchmark datasets using a linear classifier
    }
    \label{fig:sota_linear}
\end{figure*}

\begin{figure*}[t!]
    \centering
    \includegraphics[width=\linewidth]{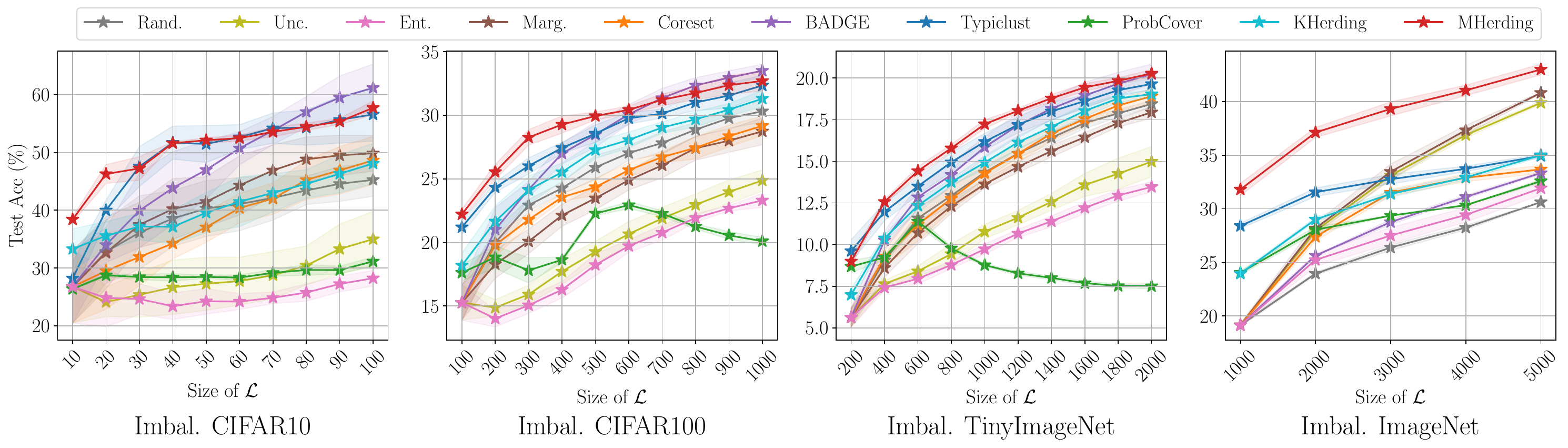}
    \caption{
    Comparison on imbalanced datasets using a linear classifier.  
    }
    \label{fig:imbal_linear}
\end{figure*}

\subsection{Comparison with Imbalanced Datasets}
\label{app:subsec:imbal}

\subsubsection{Generation of imbalanced datasets}
We generate imbalanced datasets from the balanced benchmark datasets following the standard long-tailed imbalance generation algorithm~\cite{imbal2019cui}.
Given an imbalance ratio $\rho = 0.02$, we define  $$\lvert \text{least frequent class} \rvert \defeq \lvert \text{samples per class} \rvert \times \rho.$$
Then, we randomly draw samples per class where the number of samples exponentially decays from the size of the most frequent class to the size of the least frequent class.

\subsubsection{Results}
We provide the results with the linear classifier on the imbalanced datasets in \cref{fig:imbal_linear}.
The overall trend aligns closely with the findings observed using the 1-NN classifier.
MaxHerding consistently matches or outperforms other methods, followed by Typiclust except on ImageNet.
Notably, Margin and Uncertainty show rapid improvement, surpassing Typiclust's performance after processing 3,000 data points.

\newcolumntype{D}{>{\centering\arraybackslash}p{4.8em}}
\begin{table}[t!]
\centering
\fontsize{9.0}{12.0}\selectfont
\begin{tabu}{c|D|D|D|D|D}
\hline
\multirow{2}{*}{Method}  & \multicolumn{5}{c}{Size of $\mathcal{L}$} \\
\cline{2-6}
& $200$ & $400$ & $600$ & $800$ & $1{,}000$  \\
\hline
\hline
Random & 4.4 $\pm$ 0.3 & 11.9 $\pm$ 0.7 & 20.7 $\pm$ 1.0 & 27.2 $\pm$ 0.7 & 33.4 $\pm$ 1.0 \\
Coreset & 4.4 $\pm$ 0.3 & 11.2 $\pm$ 0.8 & 19.3 $\pm$ 0.9 & 26.4 $\pm$ 0.8 & 33.2 $\pm$ 0.9 \\ 
Typiclust & 4.6 $\pm$ 0.2 & 10.9 $\pm$ 0.9 & 18.5 $\pm$ 0.6 & 27.2 $\pm$ 0.5 & 32.5 $\pm$ 0.7  \\
MaxHerding & \bftab{5.6} $\pm$ \bftab{0.1} & \bftab{13.4} $\pm$ \bftab{0.7} & \bftab{21.5} $\pm$ \bftab{0.8} & \bftab{28.9} $\pm$ \bftab{0.5} & \bftab{34.9} $\pm$ \bftab{0.6}   \\
\Xhline{2\arrayrulewidth}
\end{tabu}
\vspace{3mm}
\caption{Comparison in test performance using mAP$50$-$95$ with YOLOv5 on VOC.}
\label{app:tbl:voc}
\end{table}

\section{Low-Budget Active Learning for Object Detection}

In addition to the conventional image classification tasks, we further verify the effectiveness of the proposed MaxHerding compared to the other low-budget active learning methods: Coreset and Typiclust, for object detection tasks.
To this end, we use \href{https://github.com/ultralytics/yolov5}{YOLOv5} on PASCAL VOC dataset~\cite{voc2015Everingham}.
For the active learning methods, we use features from the penultimate layer.
\footnote{We do not include uncertainty-based methods since it is not obvious to convert the measure of uncertainty for object detection tasks.}

\cref{app:tbl:voc} reports the mean and standard deviation of five runs in the mAP$50$-$95$ metric with different size of labeled sets $\mathcal{L}$.
Interestingly, both Coreset and Typiclust are worse than Random, but the proposed MaxHerding generally outperforms Random.

To the best of our knowledge, low-budget acitve learning has not been well-studied for other tasks except image classification. 
It would be an interesting future work to compare different low-budget active learning methods for various tasks including object detection and semantic segmentation.

\end{document}